\documentclass{article} 
\usepackage{iclr2026_conference,times}

\usepackage[utf8]{inputenc} 
\usepackage[T1]{fontenc}    
\usepackage[colorlinks=true,citecolor=purple,linkcolor=purple,hypertexnames=false]{hyperref}       
\usepackage{url}
\usepackage{booktabs}
\usepackage{multicol}
\usepackage{amsmath,mathtools} 
\usepackage{amssymb} 
\usepackage{bm} 
\usepackage{amsthm}

\usepackage{amsfonts}       
\usepackage{nicefrac}       
\usepackage{microtype}      
\usepackage{xcolor}         
\usepackage{thm-restate}
\usepackage{todonotes}
\usepackage{newfile}
\usepackage{stmaryrd}
\usepackage[capitalize,nameinlink]{cleveref}
\usepackage[inline]{enumitem}
\usepackage{algorithm}
\usepackage{algpseudocode}

\usepackage{pgfplots}

\usepackage{tikz}

\usepackage{amsmath,amsfonts,bm}

\def\eqref#1{equation~\ref{#1}}

\def\1{\bm{1}}

\DeclareMathAlphabet{\mathsfit}{\encodingdefault}{\sfdefault}{m}{sl}
\SetMathAlphabet{\mathsfit}{bold}{\encodingdefault}{\sfdefault}{bx}{n}

\newcommand{\softmax}{\mathrm{softmax}}

\newcommand{\revised}[1]{{{#1}}}

\newif\ifdraft\drafttrue
\ifdraft
\newcommand\anthony[1]{{\color{blue}
\small [#1 - \textbf{Anthony}]}}
\newcommand\georg[1]{{\color{brown}
\small [#1 - \textbf{Georg}]}}

\else
\newcommand\anthony[1]{}
\newcommand\georg[1]{}

\fi

\newtheorem{theorem}{Theorem}[section]
\newtheorem{proposition}[theorem]{Proposition}
\newtheorem{lemma}[theorem]{Lemma}

\newtheorem*{question}{Research Question}
\newtheorem*{example}{Example}
\theoremstyle{definition}
\newtheorem*{remark}{Remark}
\newtheorem{definition}[theorem]{Definition}
\crefname{theorem}{Theorem}{Theorems}
\crefname{lemma}{Lemma}{Lemmas}
\crefname{proposition}{Proposition}{Propositions}
\crefname{corollary}{Corollary}{Corollaries}
\crefname{definition}{Definition}{Definitions}
\Crefname{theorem}{Theorem}{Theorems}
\Crefname{lemma}{Lemma}{Lemmas}
\Crefname{proposition}{Proposition}{Propositions}
\Crefname{definition}{Definition}{Definitions}

\newcommand\OMIT[1]{}{}

\newcommand{\ialphabet}{\Sigma}

\newcommand{\SQRT}{\mathsf{SQRT}}

\newcommand{\bt}{\bm{t}}
\newcommand{\ba}{\bm{a}}
\newcommand{\bbb}{\bm{b}}
\newcommand{\bk}{\bm{k}}
\newcommand{\bq}{\bm{q}}
\newcommand{\bx}{\bm{x}}
\newcommand{\by}{\bm{y}}
\newcommand{\bu}{\bm{u}}
\newcommand{\bv}{\bm{v}}
\newcommand{\bw}{\bm{w}}

\newcommand{\be}{\bm{e}}
\newcommand{\bzero}{\bm{0}}

\newcommand{\lta}{\mathtt{a}}
\newcommand{\ltb}{\mathtt{b}}
\newcommand{\ltc}{\mathtt{c}}
\newcommand{\ltd}{\mathtt{d}}
\newcommand{\lte}{\mathtt{e}}
\newcommand{\ltf}{\mathtt{f}}
\newcommand{\N}{\mathbb{N}}

\newcommand{\R}{\mathbb{R}}
\newcommand{\Z}{\mathbb{Z}}
\newcommand{\Q}{\mathbb{Q}}
\newcommand{\autA}{\mathfrak{A}}
\newcommand{\autB}{\mathfrak{B}}

\newcommand{\emb}{\iota}
\DeclareMathOperator{\ReLU}{ReLU}

\newcommand{\LTL}{\mathsf{LTL}}
\newcommand{\Count}{\mathsf{Count}}

\newcommand{\PARITY}{\mathsf{PARITY}}
\newcommand{\QFPA}{\mathsf{QFPA}}

\newcommand{\Parikh}{\Psi}

\newcommand{\MAJ}{\mathsf{MAJ}}
\newcommand{\X}{\operatorname{X}}
\newcommand{\U}{\operatorname{U}}

\mathchardef\mhyphen="2D 
\DeclareDocumentCommand{\AHATem}{o}{\IfNoValueTF{#1}{\mathsf{NoPE\mhyphen AHAT}}{\mathsf{NoPE\mhyphen AHAT}[\mathsf{#1}]}}
\DeclareDocumentCommand{\SMATem}{o}{\IfNoValueTF{#1}{\mathsf{NoPE\mhyphen SMAT}}{\mathsf{NoPE\mhyphen SMAT}[\mathsf{#1}]}}
\DeclareDocumentCommand{\AHATnem}{o}{\IfNoValueTF{#1}{\mathsf{NoPE\mhyphen AHAT}^{\neg\mathsf{em}}}{\mathsf{NoPE\mhyphen AHAT}^{\neg\mathsf{em}}[\mathsf{#1}]}}
\DeclareDocumentCommand{\PEAHAT}{o}{\IfNoValueTF{#1}{\mathsf{AHAT}}{\mathsf{AHAT}[\mathsf{#1}]}}

\newcommand{\SMAT}{\mathsf{SMAT}}
\newcommand{\AHAT}{\mathsf{AHAT}}

\newcommand{\SemiAlg}{\mathsf{SemiAlg}}
\newcommand{\sem}[1]{\llbracket #1\rrbracket}

\DeclareDocumentCommand{\aha}{o}{\IfNoValueTF{#1}{\mathrm{aha}}{\mathrm{aha}[\mathrm{#1}]}}

\newcommand{\weight}{\texttt{wt}}

\newcommand{\RE}{\mathsf{RE}}
\newcommand{\PI}{\mathsf{PI}}
\newcommand{\Proj}{\mathsf{Proj}}

\newcommand{\Lang}{L}

\newcommand{\defn}[1]{\emph{#1}}

\title{The Counting Power of Transformers}

\author{%
  Marco Sälzer\\
  RPTU Kaiserslautern-Landau\\
  Kaiserslautern, Germany\\
  \texttt{marco.saelzer@rptu.de}\\
  \And
  Chris Köcher\\
  MPI-SWS\\
  Kaiserslautern, Germany\\
  \texttt{ckoecher@mpi-sws.org}\\
  \And
  Alexander Kozachinskiy\\
  Centro Nacional de Inteligencia Artificial\\
  Santiago, Chile\\
  \texttt{alexander.kozachinskyi@cenia.cl}\\
  \And
  Georg Zetzsche\\
  MPI-SWS\\
  Kaiserslautern, Germany\\
  \texttt{georg@mpi-sws.org}
  \And
  Anthony Widjaja Lin\\
  MPI-SWS and RPTU Kaiserslautern-Landau\\
  Kaiserslautern, Germany\\
  \texttt{awlin@mpi-sws.org}\\
}

\newcommand{\sparagraph}[1]{\vspace{-0.3cm}\paragraph{#1}}
\iclrfinalcopy 
\begin{document}

\maketitle


\begin{abstract}
    Counting properties (e.g. determining whether certain tokens occur more
    than other tokens in a given input text) have played a significant role in 
    the study of expressiveness of transformers. In this paper, we provide a 
    formal 
    framework for investigating the counting power of transformers. We argue 
    that all existing results demonstrate transformers' expressivity only for 
    (semi-)linear counting properties, i.e., which are expressible as a 
    boolean combination of linear inequalities. 
    Our main result is that transformers can express counting properties that
    are highly nonlinear. More precisely, we prove that transformers can
    capture all semialgebraic counting properties, i.e., expressible as 
    a boolean combination of arbitrary multivariate polynomials (of any degree).
    Among others, these generalize the counting properties that
    can be captured by C-RASP softmax transformers, which capture only
    linear counting properties.

    To complement this result, we exhibit a natural subclass of (softmax) 
    transformers that completely characterizes semialgebraic counting 
    properties. 
    Through connections with the
    Hilbert's tenth problem, this expressivity of transformers also 
    yields a new undecidability result for analyzing an extremely simple 
    transformer model --- surprisingly with neither positional encodings 
    (i.e. NoPE-transformers) nor masking.
    We also experimentally validate trainability of such counting
    properties.
\end{abstract}

\section{Introduction}\label{sec:introduction}

Transformers \citep{Vaswani17} have emerged in recent years as a powerful model
with a plethora of successful applications including (among others) natural 
language processing, computer vision, and speech recognition. Despite the
success of transformers, the question of what transformers can express is still
not well-understood and has in recent years featured in a rich body of 
research works (e.g.\ \cite{transformers_survey,hahn20,perez,HAF22}). 
In particular, formal language theory provides a formal framework in
understanding expressivity issues for sequential models like transformers and
Recurrent Neural Networks (RNNs).

One recurring theme when studying the expressibility of transformers is the
\emph{counting power} of transformers. Intuitively, counting amounts to 
asserting an arithmetic relationship between the numbers of occurrences 
of various tokens in a given text. Counting properties are essentially the class
of properties for textual data under consideration in the well-known 
\emph{Vector Space Model (VSM)} (cf.\ \cite{vsm,gvsm,kernel-book}),
or the similar \emph{Bag-of-Words (BoW)} 
model \citep{bow}, \revised{which are known from the information retrieval
community to be surprisingly powerful in measuring text similarity (e.g. see
\cite{SLY19,kernel-book}).}
A simple example of a counting property 
can be found in a sentiment analysis application\footnote{\url{https://medium.com/data-science/sentiment-analysis-with-text-mining-13dd2b33de27}
}: the number of positive words exceeds the number of negative words in a text. 
In the formal language theory, such a counting property can be 
formalized as the following language 
\begin{equation}
	\MAJ := \{ w \in \{a,b\}^* : |w|_a > |w|_b \}, \label{definition-maj}
\end{equation}
which is often referred to as \emph{majority}. Here, $|w|_a$ (resp. $|w|_b$)
refers to the number of occurrences of $a$ (resp. $b$) in the string $w$. 
For example, $\lta\lta\ltb \in \MAJ$ but $\lta\ltb\ltb \notin \MAJ$.
Note that ``tokens'' in NLP are synonymous to ``letters'' in formal language
theory. Another counting
property that plays an important role in the theory of expressibility of
transformers is \emph{parity} language: 
\begin{equation}
    \PARITY := \{ w \in \{a,b\}^* : a \text{ occurs an even number of times in
	} w\}.\label{definition-parity}
\end{equation}
Multiple theoretical and empirical results (e.g.\ 
\cite{sensitivity,CC22,framework,hahn20,HAF22,BAG20,anil22,chomsky}) have shown 
that, while transformers can be efficiently trained for $\MAJ$, this is not the 
case for $\PARITY$. Several theoretical explanations have been offered, e.g.,
\emph{sensitivity} by \cite{sensitivity} and length generalization admitted by 
limit transformers by \cite{framework}). 

Thus far, existing results have touched only upon \emph{semilinear} counting 
properties. For example, defining $\MAJ$ requires only a linear inequality 
(i.e.\ $|w|_a > |w|_b$). In fact, logical languages, which were devised by 
\cite{barcelo2023logical,YC24,framework} epitomizing languages expressible
by transformers, permit only linear expressions (e.g.\ $|w|_a + |w|_b > 
2\cdot|w|_c$). However, polynomial expressions (cf.\ \cite{kernel-book}) are also 
used to express \emph{co-occurrence} of terms/tokens in a text. For example, 
using a \emph{higher-degree} monomial such as
\[
    \#(\text{nvidia}) \cdot \#(\text{intel}) \cdot \#(\text{deal}),
\]
where $\#(w)$ counts the number of occurrences of a word $w$ in the text,
one can emphasize the co-occurrence of ``nvidia'', ``intel'' and ``deal'' in a 
text. This motivates the following question:
\begin{question}
    What counting properties are expressible on transformers? Can they express
    nonlinear counting properties?
\end{question}

The main contribution of this paper is the following result.
\begin{theorem} 
    Transformers can capture all semialgebraic counting properties, i.e.,
    those expressible as a boolean combination of inequalities between
    multivariate polynomials, where each variable counts the number of 
    occurrences of a specific token in the text.
    \label{th:semialgebra-informal}
\end{theorem}
This means that transformers can capture expressions involving higher-degree
polynomials like $7\#(\text{nvidia}) \cdot \#(\text{intel}) \cdot 
\#(\text{deal}) + 2\#(\text{shares}) - 8\#(\text{war}) > 10$,
or boolean combinations (i.e. unions/intersections) of similar polynomial 
expressions. Consequently, by the Weierstrass theorem it follows that the set of
polynomials can also approximate any continuous function on the number of 
occurrences of tokens. We prove this theorem (using 
softmax transformers) --- requiring the use of neither positional encodings 
nor positional masking --- and experimentally validate this claim.

Our next question concerns the expressivity
of softmax transformers for capturing counting properties: \emph{which class of 
softmax transformers capture semialgebraic counting properties?} To this end, we
provide a surprising characterization involving \emph{average hard
attention} \citep{HAF22,perez}, which was devised to ``approximate'' soft 
attention by attending to all positions with maximum attention score and 
forwarding their average. In particular, Average Hard Attention Transformers
(AHATs) with only \emph{uniform layers} (written AHAT[U]) --- that is, where
maximum attention score is achieved at every position --- immediately form a 
subclass of SoftMax Attention Transformers (SMAT). In the sequel, we write
$\AHATem$ (resp. $\AHATem[U]$) to mean AHAT (resp. AHAT[U]) that do not use
Positional Encodings (PEs) (also no positional masking).
\begin{theorem}\label{main-result-semialgebraic}
	$\AHATem$ and $\AHATem[U]$ capture precisely semialgebraic counting 
    properties.  In particular, as far as expressing counting properties, 
    $\AHATem$ is a subset of $\SMAT$.
\end{theorem}
This is surprising, since it is still a major
open problem whether AHAT are captured by SMAT \citep{YC24,hahn20,simulating}
for general (not necessarily counting) properties. 

A corollary of Theorem \ref{th:semialgebra-informal}, combined with
Matiyasevich's celebrated solution to the notorious Hilbert's 10th Problem 
\citep{Mat93}, is a kind of \emph{universality} (i.e.
Turing-completeness) of transformers. More precisely, 
any recursively enumerable counting property $P\subseteq\ialphabet^*$ can be represented in terms of a 
program that, given an input string $w \in \ialphabet^*$, feeds each string $wv$ 
(where $v \in \Gamma^*$, for some $\Gamma \cap \ialphabet = \emptyset$) into
a transformer $T$ and accepts if $T$ accepts some $wv$. In this case, we say that $P$ is a \emph{projection} of the language accepted by $T$.
In fact, we show that transformers $T$ with only two attention layers are
sufficient and necessary to achieve this result:
\begin{theorem}\label{main-result-universality}
Every recursively enumerable counting property is a projection of a language
	recognized by a $\AHATem[U]$, and thus by an $\SMAT$. Here, two attention
	layers in $\AHATem[U]$ and $\SMAT$ are sufficient.
\end{theorem}
Similarly, our results yield an undecidability result for
analyzing an extremely simple transformer model---surprisingly with neither
positional encodings 
nor masking:
\begin{theorem}\label{main-result-undecidability}
	Given a $\AHATem[U]$ or $\SMAT$ (with just two attention layers), it is
undecidable whether its language is empty.
\end{theorem}
Recent results (cf. \citep{SAL25}) require a substantially
more complex architecture to
achieve such an undecidability result, i.e., with powerful positional encoding
and average hard attention.

Finally, \emph{how do general transformers compare with other machine learning models
as far as capturing counting properties?} To this end, let us discuss two
models. First is the class of polynomial separators that can be 
generated by mapping to a higher dimension and look for a linear separator
in this higher dimension. This is a standard technique in classical machine
learning literature, where one can apply techniques like Support Vector 
Machines (SVM) (e.g. using polynomial kernel) in the \emph{Vector 
Space Model (VSM)} (\cite{vsm,gvsm}; also see Chapter 10 of 
\cite{kernel-book}). Our result shows that transformers generalize such
counting properties: not only polynomial counting properties can
be captured, but also \emph{boolean combinations} thereof. 
Second is the model called C-RASP \citep{framework}, which is a simple 
declarative language that formalizes the so-called \emph{RASP-L conjecture} 
\citep{raspl} capturing ``efficiently learnable'' properties on transformers. 
In particular, C-RASP allows only linear counting terms. We
prove that C-RASP can capture \emph{only} linear counting properties. Since our 
experiments supporting 
Theorem \ref{th:semialgebra-informal} 
reveals that counting properties like $L_k := \{ w \in \{a,b\}^+ : |w|_a^k 
\geq |w|_b \}$
are also efficiently learnable for $k \geq 2$, it follows that C-RASP
is only a partial characterization of efficiently learnable properties.

\begin{figure}
  \begin{center}
%
%
%

\begin{tikzpicture}
  \draw[fill=gray!20,draw=none] (0,1.5) [rounded corners] -- (0,2)  -- (3,2) [sharp corners] -- (3,1.5) -- cycle;
  \draw[fill=gray!20,draw=none] (3.75,1.5) [rounded corners] -- (3.75,2) -- (6.5,2) [sharp corners] -- (6.5,1.5) -- cycle;
  
  \draw[rounded corners] (0,0) rectangle (3,2);
  \draw[rounded corners] (3.75,0) rectangle (6.5,2);
  \draw[rounded corners] (7.25,0) rectangle (9.25,2);
  \draw[rounded corners] (10,0) rectangle (11.5,0.95);
  \draw[rounded corners] (10,1.05) rectangle (11.5,2);
  
  \node[align=center] at (1.5,0.75) {$\AHATem[\leq1]$\\$\QFPA$};
  \node[align=center] at (5.125,0.75) {$\AHATem$\\$\AHATem[U]$\\$\SemiAlg$};
  \node[align=center] at (8.25,1) {$\PEAHAT[U]$};
  \node[align=center] at (10.75,0.5) {$\SMAT$};
  \node[align=center] at (10.75,1.5) {$\PEAHAT$};
  
  \node at (3.375,1) {$\subsetneq$};
  \node[align=center] at (6.875,1) {$\subseteq$};
  \node[align=center] at (9.625,0.5) {$\subseteq$};
  \node at (9.625,1.5) {$\subseteq$};
  
  \node at (1.5,1.75) {\cref{result-end-marker-one-layer}};
  \node at (5.125,1.75) {\cref{main-result-semialgebraic}};
\end{tikzpicture}
  \end{center}
  \caption{\revised{Visualization of our results.\label{fig:results}}}
\end{figure}

\paragraph{Organization.} 
We recall transformer models and define our framework for studying counting
properties in \cref{sec:model}.
We then show how to capture semialgebraic counting properties using
transformers in \cref{sec:semialgebraic}.
In \cref{sec:characterization}, we 
provide a natural subclass of softmax transformers that completely 
characterizes semialgebraic counting properties. 
In \cref{sec:apps}, we show applications of our semialgebraic results for a
better understanding of expressiveness of transformers, e.g.,
universality/undecidability and comparison to work on C-RASP transformers.
We report our experimental results in \cref{sec:exps} and
conclude in \cref{sec:conc}. Some details have been relegated into the Appendix.

\section{Framework: Transformers and Counting Properties}\label{sec:model}

\paragraph{Formal language theory primer}
We assume some basic understanding of formal language theory
(at the level of a standard undergraduate textbook by \cite{sipser-book})
and will only fix some notation.

For an alphabet $\ialphabet=\{\lta_1,\ldots,\lta_m\}$. A \defn{language}
is a set of strings over $\ialphabet$. We write $\ialphabet^*$
(resp. $\ialphabet^+$) to mean the set of all strings (resp. all nonempty
strings) over $\ialphabet$. We write $|w|$ to denote the length of $w$. For each $a \in
\ialphabet$, we write $|w|_a$ to mean the number of occurrences of $a$ in $w$.
A language $K\subseteq\ialphabet^*$ is a \emph{projection} of a language $L\subseteq\ialphabet^*$ if there is a subalphabet $\Gamma\subseteq\ialphabet$ such that $K$ is obtained from $L$ by deleting all occurrences of letters in $\Gamma$ from words in $L$. For a class $\mathcal{C}$ of languages, by $\Proj(\mathcal{C})$, we denote the class of projections of languages in $\mathcal{C}$.

We will touch upon regular languages and recursively enumerable languages (see
\cite{sipser-book} for details). In summary, regular languages are languages
that can be described by regular expressions.  \emph{Recursively enumerable}
languages are those that are recognized by (possibly nonterminating) Turing
machines. The class of such languages is denoted $\RE$.  In particular, a
machine model is said to be \emph{Turing-complete} if it can capture all
recursively enumerable languages.

For an alphabet $\ialphabet=\{\lta_1,\ldots,\lta_m\}$, we define the
\emph{Parikh image} (a.k.a. \emph{Parikh map}) as the function $\Psi\colon\Sigma^*\to\N^m$, where $\Psi(w)[i]:=|w|_{\lta_i}$ is the number of $\lta_i$'s in $w$. 
Intuitively, Parikh image of a word $w$ provides the letter counts in $w$, e.g.,
over $\ialphabet = \{\lta,\ltb\}$, we have $\Psi(abaa) = (3,1)$. The Parikh map
can also be extended to a language $L$; that is, $\Psi(L) = \{ \Psi(w) : w \in 
L \} \subseteq \N^{|\ialphabet|}$. For example, if $L = \{ \lta^n\ltb^n\lta^n : 
n \geq 0\}$ is a language over $\ialphabet = \{\lta,\ltb\}$, we have $\Psi(L) = 
\{ (2n,n) : n \geq 0 \}$.

\subsection{Transformers}
We now recall the formal definition of transformers. Loosely speaking,
a transformer is a composition of finitely many attention layers, each
converting a sequence $\sigma$ of $\R^d$-vectors into another sequence 
$\sigma'$ of $\R^k$-vectors, for some $d$ and $k$. To turn a transformer $T$ into a language recognizer, we have to embed
any letter in the finite alphabet $\ialphabet$ as a $\R^d$-vector, where
$d$ is smaller than the dimension of the first attention layer. For example, 
$\ialphabet = \{\lta,\ltb,\ltc\}$, and the \emph{one-hot} embeddings of $\lta$, $\ltb$, $\ltc$
are (respectively) $(1,0,0)$, $(0,1,0)$, and $(0,0,1)$. Finally, to determine
acceptance, we simply run $T$ on the embeddings of the input string $w$ into
a sequence of vectors (possibly expanded with positional information) and check
if the last vector $\bv$ satisfies that the dot product $\bv.\bt$ is greater
than $0$ (for 
some pre-defined vector 
$\bt$ of weights). In particular, $w$ is accepted by $T$ iff $\bv.\bt > 0$.

\begin{example}
Suppose we are given the input string $w = \lta\ltb\lta\ltc$. Additionally, suppose we use
the positional embedding $p\colon n \mapsto 1/n$. Then, checking whether $T$ accepts
$w$ amounts to running $T$ on the sequence $\sigma$:
\[
    (1,0,0,1)(0,1,0,1/2)(1,0,0,1/3)(0,0,1,1/4).
\]
	After running $T$ on $\sigma$, the resulting sequence is of the form
    $\bv_1,\bv_2,\bv_3,\bv_4$.
Determining whether $T$ accepts $w$ amounts to checking whether $\bt.\bv_1 > 0$.
For example, $\bv_1,\bv_2,\bv_3,\bv_4$ could be:
    \[
        (1,1,7,1,1)(2,3,1,10,1/2)(1,8,0,8,1/3)(0,0,1,-1,1/4)
    \]
which will be accepted, whenever $t = (1,0,0,1,0)$.
\end{example}

Next we formalize the definition of transformers by defining how each attention
layer functions.

\paragraph{ReLU networks.}
We first define ReLU networks, which are used inside an attention layer.
A \emph{ReLU node} $v$ is a function $\Q^m \to \Q$, where $m \in \N$ is referred to as the input dimension, and is defined as
$v(x_1, \dotsc, x_m) = \max(0, b + \sum_{i=1}^n w_i x_i)$, where $w_i \in \Q$ are the \emph{weights}, and
$b \in \Q$ is the \emph{bias}. \revised{[In practice, GeLU and SwiGLU are also used
instead of ReLU, which we do not consider in this paper.]}
    A \emph{ReLU layer} $\ell$ is a tuple of ReLU nodes $(v_1, \dotsc, v_n)$, all having the same
input dimensionality, computing a function $\R^m \to \R^n$, where $n \in 
\N$ is referred to as the output dimension.
Finally, a \emph{ReLU network} $\mathcal{N}$ is a tuple of ReLU layers $(\ell_1, \dotsc, \ell_k)$, such that the input dimension of $\ell_{i+1}$ is equal to
the output dimension of $\ell_i$. It computes a function $\Q^{m_1} \to \Q^{n_k}$, given by
$\mathcal{N}(x_1, \dotsc, x_{m_1}) = \ell_k(\dotsb \ell_1(x_1, \dotsc, x_{m_1}) \dotsb )$.

\paragraph{Attention layers}
Each attention layer involves a \emph{weight normalizer} $\weight: 
\R^* \to \R^*$, which turns any $d$-sequence of weights into another 
such $d$-sequence. Two widely used weight normalizers are:
    \begin{enumerate}
        \item The softmax normalizer $\softmax$. That is, given a sequence 
            $\sigma = x_1,\ldots,x_n \in \R$, define $\softmax(\sigma) 
            := y_1,\ldots,y_n$, where
            $y_i := \frac{e^{x_i}}{\sum_{j=1}^n e^{x_j}}$.
        \item The averaging hard attention normalizer $\aha$. We define
            $\aha(\sigma) := y_1,\ldots,y_n$, where
            \[
                y_i := \left\{ \begin{array}{cc} 
                    1/|P| & \text{ if $x_i = \max(\sigma)$, } \\
                    0     & \text{ or else. }
                            \end{array}
                            \right.
            \]
            where $P$ consists of positions $i$ in $\sigma$ such that $x_i$
            is maximum in $\sigma$. That is, $\aha$ behaves like $\softmax$ 
            but maps all non-maximum weights to 0, and all maximum weights to 
            $1/|P|$.
    \end{enumerate} 
One can also allow a temperature scaling $\tau > 0$ to $\softmax$,
    i.e., $\softmax_{\tau}(\sigma) = y_1,\ldots,y_n$ and set
            $y_i := \frac{e^{x_i/\tau}}{\sum_{j=1}^n e^{x_j/ \tau}}$.
    This is not so relevant in our paper since our proof works for \emph{any} 
    $\tau > 0$.

%


    An \emph{attention layer} is a function $\lambda\colon(\R^d)^* \to(\R^e)^* $, given by affine maps $Q, K\colon \R^d\to\R^m$, $V\colon\R^d\to \R^k$ (query, 
    key, and value matrices) and a ReLU neural net
    $\mathcal{N}\colon\Q^{d+k}\to\Q^e$.
  Given an input sequence $x = (\bx_1, \ldots, \bx_n)\in(\Q^d)^n$, the output
    sequence  $y = (\by_1, \ldots, \by_n)\in(\Q^d)^n$ is computed as follows. First, one computes the sequences of key, query, and value vectors:
  $\bk_i = K \bx_i, \,\, \bq_i = Q \bx_i, \,\, \bv_i = V \bx_i$, for each 
    $i=1, \ldots, n$, then we define $\by_i = \mathcal{N}(\bx_i, \ba_i)$,
    with $\ba_i = \sum_{j=1}^n \bw(j)\bv_j$, where $\bw = \weight(\{\langle
    \bk_i,\bq_j\rangle\}_{j=1}^n)$. 

    We say that $\lambda$ is a \emph{softmax} (resp. \emph{aha}) layer if
    $\weight = \softmax$ (resp. $\aha$). We say that it is a 
    \emph{uniform-$\aha$} layer if it is an $\aha$ layer such that $K\bx = 
    Q\bx = \bzero$ for all $\bx$, i.e.,
    $\langle K\bx,Q\by\rangle = 0$ for all $\bx$ and $\by$. Note that a 
    uniform-$\aha$ is both an $\aha$ layer and a $\softmax$ layer since
    noting that 
    \[
        \softmax(s_1,\ldots,s_n)=\softmax_\tau(s_1,\ldots,s_n) = 
        \aha(s_1,\ldots,s_n) = [1/n, \cdots, 1/n], 
    \]
    whenever $s_1 = \cdots = s_n$, which can be guaranteed 
    for uniform $\aha$ layers. This holds for \emph{all} $\tau > 0$.

\begin{remark}
        Some papers (e.g.\ \cite{ACY24,framework,YC24}) apply \emph{strict future
        masking}, which means that attention is only applied to positions
        up to the current position $i$. Our work does not apply masking.
\end{remark}

\paragraph{Defining transformers.}
To define a transformer and its language, we first extend
the finite alphabet $\ialphabet$ with an 
\emph{end marker} $\$\notin\Sigma$. That is, $\Gamma := \Sigma \cup \{\$\}$.
A \emph{transformer} with $\ell$ layers over a finite alphabet $\ialphabet$
is then a function $T\colon\Sigma^+ \to \{0, 1\}$, given by:
  (i) the ``input embedding''  function $\emb\colon\Gamma\to\mathbb{Q}^{d_1}$,
  (ii) the positional encoding $p\colon\N^2\to\R^{d_1}$, and
  (iii) a sequence of layers $\lambda_1\colon(\R^{d_1})^*\to (\R^{d_2})^*, 
          \ldots, \lambda_\ell\colon (\R^{d_\ell})^*\to (\R^{d_{\ell+1}})^*$.
  Given an input word $w = a_1\cdots a_n\in\Sigma^n$, the output $T(w)$ is computed as follows. First, we set
$	  \bx_1 = \emb(a_1)+p(n+1,1),~ \ldots,~ \bx_n = \emb(a_n)+p(n+1,n), ~\bx_{n+1} =
\emb(\$)+p(n+1,n)$.
  Then we compute
$	  (\by_1, \ldots, \by_{n+1}) = \lambda_\ell(\lambda_{\ell-1}(\cdots
\lambda_1(\bx_1, \ldots, \bx_{n+1})\cdots))$,
	and we set $T(w) = 1$ if and only if $\by_n[1] > 0$, and $T(w)=0$ otherwise.
    The language $\Lang(T)$ accepted by $T$ is defined as $\{ w \in \ialphabet^*
    \colon T(w) = 1\}$.
    We say that $T$ has \emph{no positional encoding} (NoPE) if the positional
    encoding is a constant function. 

\revised{
    \begin{remark}
        Several studies (e.g., \cite{MerrillS23,SAL25,LiC25}) consider the capabilities of
        transformers in the context of restricted precision, such as assuming computations
        are carried out under the assumption of finite representation sizes. We do not focus on 
        these aspects, but note that it is easy to see that our key results, such as \cref{polynomial-inequality},
        also apply under so-called \emph{log-precision} assumptions (cf.\
        \cite{MerrillS23}; also see \cite{MS23-nips})
        for rational numbers. This means that the binary representation size of a number
        $p/q \in \mathbb{Q}$ grows logarithmically with the length of the input.
    \end{remark}
}
  \revised{A \emph{Softmax Attention Transformer} is a transformer using only $\softmax$ layers
  whereas an \emph{AHA Transformer} is a transformer using only $\aha$ layers.
  By $\SMAT$ we denote the class of all languages accepted by softmax attention
  transformers and by $\PEAHAT$ we denote the class of all languages accepted by AHA
  transformers. To all classes we of transformer languages we append ``$[\mathsf{U}]$'' to
  denote languages of transformers with only uniform layers, e.g.\ $\PEAHAT[U]$. We prepend
  ``$\mathsf{NoPE}$'' to denote only languages of transformers with no positional encoding,
  e.g.\ $\AHATem[U]$.
  Note that all transformer models we are considering in this paper have only one attention head.}

\subsection{Counting Properties}
We now define a framework for studying the counting ability of transformers.
Intuitively, our framework focuses on ``counting properties''. As we
shall see below, we can build many interesting formal languages with the help 
of purely counting properties.

Given a permutation $\pi: \{1,\ldots,n\} \to \{1,\ldots,n\}$ and a string
$w = w_1 \cdots w_n$ of length $n$, the string $\pi(w) := w_{\pi(1)}\cdots
w_{\pi(n)}$ is obtained by permuting the letters in $w$ according to $\pi$.
\begin{definition}
    A \defn{counting property} over the alphabet $\ialphabet$ is a 
    permutation-closed language $L$, i.e., for 
    each $w \in \ialphabet^*$, it is the case that $w \in L$ iff $\pi(w) \in L$
    for \emph{each} permutation $\pi$ over $\{1,\ldots,|w|\}$.
\end{definition}
Examples of counting properties are $\MAJ$ and $\PARITY$ (see (\ref{definition-maj}), (\ref{definition-parity})).
We often identify a counting property $L$ with its set $\Psi(w) 
\subseteq \N^{|\ialphabet|}$ of letter counts (i.e. Parikh image).
By $\PI$, we denote the class of counting properties over 
$\ialphabet$. Counting properties are also called \defn{permutation-invariant} or
    ``proportion-invariant'' languages, e.g., see 
    \cite{perez,barcelo2023logical}.

\paragraph{Why counting properties?} Certainly, many languages of interests 
have both a ``counting component'' and an ``order component''. Take, for
example, the language $L_1 = \{ \lta^n\ltb^n\ltc^n : n \geq 0 \}$. Our framework focuses
on \emph{purely} counting properties for two reasons. Firstly, it abstracts away
non-counting components that cannot be captured by the model. Secondly,
many formal languages $L$ of interests can be constructed by taking intersection
of a counting property $P$ and an order (and counting-insensitive) language
$L'$. For example, $L_1$ above can be written as
$P \cap L'$, where $P = \{ w \in \ialphabet^* : |w|_{\lta} = |w|_{\ltb}  = |w|_{\ltc} \}$ and
$L' = \lta^*\ltb^*\ltc^*$. Finally, multiple key languages 
in the literature on the expressivity of transformers are in fact counting
properties (e.g. $\MAJ$ and $\PARITY$).



\section{Capturing Semialgebraic Counting Properties}
\label{sec:semialgebraic}
A subset $S\subseteq\N^m$ is \emph{semi-algebraic} if it is a Boolean
combination of sets of the form $S_p=\{\bx\in\N^m \mid p(\bx)>0\}$ for some
polynomial $p\in\Z[X_1,\ldots,X_m]$. A language $L\subseteq\Sigma^*$ is
\emph{semi-algebraic} if there is a semi-algebraic set $S\subseteq\N^m$ and
$\Sigma=\{\lta_1,\ldots,\lta_m\}$ such that $L=\{w\in\{\lta_1,\ldots,\lta_m\}^*
\mid \Psi(w)\in S\}$. Let $\SemiAlg$ denote the class of semi-algebraic
languages. An example is
\begin{equation}
	\SQRT=\{w\in\{\lta,\ltb\}^* \mid |w|_{\lta}<
    |w|/\sqrt{2}\},\label{sqrt-language},
\end{equation}
since $|w|_{\lta}<|w|/\sqrt{2}$ if and only if $2|w|_{\lta}^2<|w|^2$.
Likewise, extending the coefficients of our polynomials to rational numbers
does not increase the expressiveness of semialgebraic sets, e.g., $\tfrac{7}{3}xy + y^2 
> 8x - 3$ can be rewritten as $7xy + 3y^2 > 24x - 9$. 
Note that for every $p\in\Z[X_1,\ldots,X_m]$, the set $\{\bx\in\N^m \mid 
p(\bx)=0\}$ is semi-algebraic, because $p(\bx)=0$ if and only if 
$-p(\bx)^2+1>0$. Thus, every solution set to polynomial equations is also
semi-algebraic.

We show Theorem \ref{th:semialgebra-informal}. Since $\PEAHAT[U]\subseteq\SMAT$, it sufices to construct a $\PEAHAT[U]$. We will even construct a $\AHATem[U]$. The key ingredient is:
\begin{proposition}\label{polynomial-inequality}
	For every polynomial $p\in\Z[X_1,\ldots,X_m]$, the language
	$L_{p>0}=\{w\in\{\lta_1,\ldots,\lta_m\}^* \mid p(\Psi(w))>0\}$ belongs
    to $\AHATem[U]$. Thus, $L_{p>0}$ is in SMAT. 
\end{proposition}
Let us see why \cref{polynomial-inequality} implies
$\SemiAlg\subseteq\AHATem[U]$.  First, the complement of each language
$L_{p>0}$ can be obtained, because $p(\bx)>0$ is violated if and only if
$-p(\bx)+1>0$. Moreover, $\AHATem$ is closed under union and intersection (we prove a stronger fact in \cref{app:union-intersection}). We can thus accept all Boolean combinations of
languages of the form $L_{p>0}$, and hence $\SemiAlg$.

To show \cref{polynomial-inequality}, we will use polynomials that are \emph{homogeneous},
meaning all monomials have the same degree. Note that given an arbitrary
polynomial $p\in\Z[X_1,\ldots,X_m]$ of degree $d$, we can consider the
polynomial $q\in\Z[X_0,\ldots,X_m]$ with
$q=X_0^dp(\tfrac{X_1}{X_0},\ldots,\tfrac{X_m}{X_0})$, which is homogeneous. It
has the property that $p(x_1,\ldots,x_m)>0$ if and only if
$q(1,x_1,\ldots,x_m)>0$. Therefore, from now on, we assume that we have a
homogeneous polynomial $q\in\Z[X_0,\ldots,X_m]$ and want to construct an AHAT[U] for the language $K_q=\{w\in\{\lta_1,\ldots,\lta_m\}^* \mid \text{$q(1,\bx)>0$ for $\bx=\Psi(w)$}\}$.

To simplify notation, we denote the end marker by $\lta_0$.
Thus, the input will be a string $w\in\{\lta_0,\ldots,\lta_m\}^+$
that contains $\lta_0$ exactly once, at the end.  Since $|w|_{\lta_0}=1$ is
satisfied automatically, our AHAT[U] only has to check that $q(x_0,\ldots,x_m)>0$,
where $x_i=|w|_{\lta_i}$.
The input encoding is the map $\{\lta_0,\ldots,\lta_m\}^*\to \Q^m$ with
$\lta_i\mapsto\be_i$, where $\be_i\in\Q^m$ is the $i$-th unit vector.

\revised{\sparagraph{Overall idea} Roughly speaking, we implement multiplication via
averaging as follows. For each letter $\lta_i$, we have a gadget
that can multiply an existing entry $y\in[0,1]$ (in each vector) by $\tfrac{x_i}{n+1}$
(recall that $n$ is the overall word length). This is done by first multiplying
the existing entries either (i)~by $1$ if the current letter is $\lta_i$ or
(ii)~by $0$ if the current letter is not $\lta_i$. This is achieved using a
ReLU layer, by observing that for $u\in[0,1]$ and $v\in\{0,1\}$, we have $u\cdot v=\ReLU(u-(1-v))$. After this, we average over the entire input in this
component. Since we make sure that all the entries we multiplied with $0$ or
$1$ had the same value $y\in[0,1]$, taking the average will result in the value
$\tfrac{y\cdot x_i}{n+1}$. Repeating this for a monomial
$x_{i_1}\cdots x_{i_d}$, we arrive at the value $\tfrac{x_{i_1}\cdots
x_{i_d}}{(n+1)^d}$.
Since our homogenization step ensured that all our monomials have the same
degree $d$, adding up the entries corresponding to the monomials will yield
$\tfrac{p(\Psi(w))}{(n+1)^d}$. Finally, the latter quantity is positive if and
only if $p(\Psi(w))> 0$.}

\sparagraph{Step I: Compute frequencies}
Our AHAT[U] first uses an attention layer to compute $m+1$ new components, where
$i$-th component holds $\tfrac{x_i}{n+1}$, where $n+1$ is the length of the
input (including the end marker). This is easily done by attending to all
positions and computing the averages of the first $m+1$ components.
To simplify notation, we will index vectors starting with index $0$.

\newcommand{\omult}{\mathsf{omult}}
\sparagraph{Step II: Multiplication gadgets}
Second, we have a sequence of gadgets (each consisting of \revised{one ReLU layer and one attention layer}) that perform the multiplication. 
Each gadget introduces a new component, and does not change the existing components.
Between gadget executions, the following additional invariants are upheld:
\begin{enumerate*}[label=(\roman*)]
\item Overall, a gadget does not change existing components: it introduces one new component.
\item The components $\{0,\ldots,m\}$ are called the \emph{initial} components.
\item All other components are \emph{uniform}, i.e.\ they are the same across all positions. 
\item The uniform components carry values in $[0,1]$.
\end{enumerate*}
Thus, we will call components $0,\ldots,m$ the \emph{initial} components; and we call components $>m$ the \emph{uniform} components.

Our gadgets do the following. Suppose we have already produced $\ell$ additional components. For each initial component $i\in[0,m]$ and uniform component $j\in[m+1,m+1+\ell]$, gadget $\omult(\ell,i,j)$, which introduces a new component, will carry the value $\frac{x_i\cdot y_j}{n+1},$
where $y_j$ is the value in component $j$ of all vectors. Recall that we use $x_i$ to denote the number of $\lta_i$ 
occurrences in the input for $i\in[0,m]$.

We implement the gadget $\omult(\ell,i,j)$ using some ReLU layers and an
attention layer. Suppose that before, we have the vector
$\bu_p\in\Q^{m+1+\ell}$ in position $p$.  First, using ReLU layers, we
introduce a new component that in position $p$ has the value $\bu_p[i]\cdot
\bu_p[j]$. This can be achieved since $\bu_p[i]$ is in $\{0,1\}$ and
$\bu_p[j]\in[0,1]$: Notice that $\bu_p[i]\cdot\bu_p[j]=\ReLU(\bu_p[j]-(1-\bu_p[i]))$.
Indeed, if $\bu_p[i]=1$, then this evaluates to $\bu_p[j]$; if
$\bu_p[i]=0$, then we get $\ReLU(\bu_p[j]-1)=0$.
We then use uniform attention to compute the average of this new
$\bu_p[i]\cdot\bu_p[j]$-component across all vectors. Since there are $n+1$
vectors, exactly $x_i$ of them have $\bu_p[i]=1$, and also $\bu_p[j]=y_j$, we get the desired
$\tfrac{x_i\cdot y_j}{n+1}$.

\sparagraph{Step III: Computing the polynomial}
We now use our gadgets to compute the value of the polynomial. For each
monomial of $q$, say $X_{i_1}\cdots X_{i_d}$, we use $d-1$ gadgets to compute
$x_{i_1}\cdots x_{i_d}/(n+1)^d$: The frequency computation in the beginning
yields $x_{i_1}/(n+1)$, and then we use gadgets to compute
$x_{i_1}x_{i_2}/(n+1)^2$, $x_{i_1}x_{i_2}x_{i_3}/(n+1)^3$, etc.\ until
$x_{i_1}\cdots x_{i_d}/(n+1)^d$. Finally, we use a ReLU layer to multiply each
monomial with a rational coefficient, and compute the sum of all the monomials.
Thus, we have computed $q(x_0,\ldots,x_m)/(n+1)^d$. We accept if and only if
$q(x_0,\ldots,x_m)/(n+1)^d>0$. Note that this is the case if and only if
$q(x_0,\ldots,x_m)>0$.

This completes \cref{polynomial-inequality} and thus $\SemiAlg\subseteq\AHATem[U]$.
\revised{
	We remark that the embedding dimension and the number of layers of our transformer in Proposition 
	\ref{polynomial-inequality} depends on the degree $d$ 
	and the number $M$ of monomials in $p$. We require at most $O(d)$ layers, each layer 
	increasing the degree of the computed monomials by one. In the appendix, we detailed that 
	polynomials of degree $d$ are accepted by $\AHATem[U]$ using at most $d$ attention layers 
	(see \cref{semialg-parametric}). The embedding dimension is $O(dM)$ because we store the value of each 
	monomial in a separate dimension.
}

\section{Characterizing semi-algebraic counting properties}
\label{sec:characterization}
We have shown that $\AHATem[U] \subseteq \SMAT$ can capture semi-algebraic
counting properties. We now prove that the subclass $\AHATem[U]$ precisely
characterizes $\SemiAlg$.

\begin{proposition}\label{ahat-to-semi-algebraic}
	$\AHATem\subseteq\SemiAlg$.
\end{proposition}
\vspace{-0.5cm}
	\begin{proof}
	Suppose that $\Sigma=\{\lta_1,\ldots,\lta_m\}$ is our alphabet,
	$\lta_0$ the end marker, and $x_i\in\N$ the number of occurrences of
	$\lta_i$ in the input.
 We say that a
	position $p$ is an \emph{$\lta_i$-position} if the input holds $\lta_i$
	at position $p$.  Notice that an AHAT without positional encoding
	cannot distinguish vectors that come from the same input letter.  This
	means, in any layer, any two $\lta_i$-positions will hold the same
	vector. Thus, the vector sequence on layer $\ell$ is described by
	rational vectors $\bu_{\ell,0},\ldots,\bu_{\ell,m}$, where
	$\bu_{\ell,i}$ is the vector at all the $\lta_i$-positions on layer
	$\ell$.  Moreover, for each $i$, the set of positions maximizing an
	attention score also either contains all $\lta_i$-positions, or none of
	them. Therefore, if the AHAT has $a$ attention layers, there are at
	most $((2^{m+1})^{m+1})^a=2^{(m+1)^2a}$ possible ways to choose the
	positions of maximal score: On each attention layer, and for each
	$i\in[0,m]$, we select a subset of the $m+1$ letters. For
	each ReLU node and each $i$, there are two ways its expression
	$\ReLU(v)$ can be evaluated: as $0$ or as $v$.  Thus, if there are $r$
	ReLU nodes, then there are $2^r$ ways to evaluate all those nodes.

	 For each of these $2^{r+(m+1)^2a}$ choices, we construct a conjunction
	 of polynomial inequalities that verify that (i)~this choice actually
	 maximized scores, (ii)~the resulting vector at the right-most position
	 in the last layer satisfies the accepting condition. This is easy to
	 do by building, for each layer $\ell$ and each $i$, expressions in
	 $x_1,\ldots,x_m$ for the vectors $\bu_{\ell,i}$, assuming our choice
	 above. These expressions have the form
	 $p(x_1,\ldots,x_m)/q(x_1,\ldots,x_m)$ (averaging can introduce
	 denominators). Here, once we have expressions for $\bu_{\ell,i}$, we
	 can use them to build expressions for $\bu_{\ell+1,i}$ by following
	 the definition of AHAT. Checking (i) and (ii) is then also easy,
	 because inequalities involving quotients
	 $p(x_1,\ldots,x_m)/q(x_1,\ldots,x_m)$ can be turned into polynomial
	 inequalities by multiplying with common denominators.
	 Finally, we take a disjunction over all $2^{r+(m+1)a}$ conjunctions. 
	\end{proof}

\sparagraph{Inexpressibility of $\PARITY$.}
Our characterization of $\AHATem$
(i.e. \cref{ahat-to-semi-algebraic}) implies an interesting inexpressibility
result regarding $\PARITY$ (see (\ref{definition-parity}):
%
%
\begin{restatable}{corollary}{ResultParity}\label{result-parity}
	$\PARITY$ does not belong to $\AHATem$.
\end{restatable}
$\PARITY$ is known to be accepted by AHAT~\citep{barcelo2023logical} and by
SMAT~\citep{CC22} (with PE). Inexpressibility of $\PARITY$ in a
length-generalizable subclass of $\SMAT$ and $\AHAT$ (with struct future
masking and positional encodings) is known \citep{framework}. Similarly,
$\PARITY$ is not expressible by $\SMAT$ with strict future
masking~\citep{hahn20}. \cref{result-parity} complements these results and is an
easy corollary of \cref{ahat-to-semi-algebraic} (see \cref{app:parity}).

\section{Applications}
\label{sec:apps}

\subsection{Universality and undecidability of transformers}
Let us discuss why universality/undecidability 
(i.e. \cref{main-result-universality,main-result-undecidability}) follow from
\cref{main-result-semialgebraic}. First, by the well-known theorem
``MRDP'' theorem \citep{Mat93} due to Matiyasevich, Robinson, Davis, and 
Putnam, every language in $\RE\cap\PI$ is a projection of a language of the form
$L_p=\{w\in\{\lta_1,\ldots,\lta_m\}^*\mid p(\Psi(w))=0\}$, where
$p\in\Z[X_1,\ldots,X_m]$ is a polynomial. Since $L_p$ belongs to $\AHATem[U]$,
we thus obtain \cref{main-result-universality}. Furthermore, since our translation from polynomials to $\AHATem[U]$ (and thus $\SMAT$) is effective, this also implies \cref{main-result-undecidability}: By the MRDP theorem (which is also effective), it is undecidable whether a given polynomial $p\in\Z[X_1,\ldots,X_m]$ has a solution. Using our translations, we can turn such a $p$ into a $\AHATem$ (or $\SMAT$) that is non-empty if and only if $p$ has a solution.

\paragraph{Using only two layers}
In fact, in \cref{main-result-universality,main-result-undecidability}, we even
claim that two layers suffice for universality and undecidability. Let us
sketch this here. First, our construction above yields a
$\AHATem[U]$ of at most $\ell$ layers, provided that the polynomials in the
semialgebraic set all have degree $\le\ell$ (see \cref{app:semialg-to-ahat}). In particular, we show that for each $\ell$, $\AHATem[\ell,U]$ is closed under union and intersection (see \cref{app:union-intersection}). Furthermore, we rely on the
well-known fact that the set of solutions of a polynomial equation $p=0$ can
always be written as the projection of the set of solutions of a \emph{system
of quadratic equations}. Since by our stronger version of
\cref{main-result-semialgebraic}, intersections of solution sets of 
quadratic equations only require a $\AHATem[U]$ with $\le 2$ layers, this yields
the stronger versions of \cref{main-result-universality,main-result-undecidability}. See
\cref{app:parametric} for details (where we also show that with just one layer, \cref{main-result-universality,main-result-undecidability} do not hold).

\subsection{Comparison with C-RASP and LTL with Counting}
C-RASP 
\citep{framework,YC24} is a simple programming language that can be converted
into softmax transformers. 
In
particular, it is a subset of the so-called \emph{LTL with Counting}
\citep{YC24,barcelo2023logical}. For example, $\{ w \in \{a,b\}^* : |w|_a = |w|_b
\}$ can be written as the following formula in LTL with Counting:
    $\overrightarrow{\#a} = \overrightarrow{\#b}$.
In particular, only linear expressions can be constructed in such formulas.
We show in the appendix that LTL with Counting (and therefore C-RASP) only
capture (semi)linear counting properties, i.e., boolean combinations of linear
inequalities (and modulo arithmetics), so not languages like  $L_k := \{ w \in \{a,b\} : |w|_a^k \geq |w|_b \}$.
\begin{proposition}
    LTL with Counting can define only (semi)linear counting properties.
\end{proposition}





\section{Experiments}
\label{sec:exps}
In this section, we experimentally complement our main result (cf.
    \cref{th:semialgebra-informal}) that
transformers can capture 
solutions of polynomial equations of higher degree.
In particular, our results suggest that softmax transformers should
be able to learn languages encoding solutions of polynomial equations. 

We test our hypothesis on extensions of $\MAJ$ with polynomial inequalities.
That is, we define the language $L_k$ is defined by $L_k = \{w \in \{a,b\}^+ 
\mid |w|_b \leq (|w|_a)^k\}$, representing
the set of solutions for the simple equation $y \geq x^k$. 
\begin{quote}
    \emph{
    Do softmax transformer classifiers perform well on language $L_k$?
    Additionally, can we observe tendencies of length-generalization?
    }
\end{quote}
In other words, the task of the transformer is 
a binary classification such that $T(w)$ accepts if $w \in L_k$ and it does not if $w \not\in L_k$.

\revised{
We train softmax encoders without positional encoding and otherwise in line with the vanilla model, introduced by
\cite{Vaswani17}, as binary classifiers using components offered by Pytorch's \texttt{nn.Module} based on a balanced dataset of $5\cdot 10^5$ data points
sampled from $L_k$ for $k = 1, \dotsc, 5$ of words up to length 500
In all experiments, we conduct a single epoch and
choosed the best model conducting early stopping based on the binary-cross entropy loss combined with softmax, the typical metric for models 
outputting a probability for binary classification, offered in a numerical stable version by Pytorch's \texttt{nn.Module} in form of 
\texttt{BCEWithLogitsLoss} , on a validation dataset sampled from the same distribution
and of the same size as the training dataset. To partially explore the hyperparameter space, we conduct a grid search
over number of layers 1 to 5, number of heads per layer 1, 2 or 4. In all experiments, we fixed the input
features to 32, the feedforward dimension to 64, the dropout rate to 0.3, and optimized using the AdamW optimizer with a learning rate of $10^{-4}$ and 
weight decay of 0.01 as, again, offered by Pytorch's \texttt{optim} package.
}

\begin{figure}[t]
    \centering
    \begin{minipage}{0.4\textwidth}
        \centering
        \begin{tabular}{cccc}
            \toprule
            $k$ & Val. Perf. & Test Perf. & Gen. Perf. \\
            \midrule
            1 & 0.015 & 0.016/0.99 & 0.301/0.95 \\
            2 & 0.024 & 0.033/0.99 & 0.324/0.94 \\
            3 & 0.023 & 0.021/0.99 & 0.299/0.96 \\
            4 & 0.019 & 0.020/0.99 & 0.099/0.97 \\
            5 & 0.020 & 0.024/0.99 & 0.107/0.96 \\
            \bottomrule
        \end{tabular}
    \end{minipage}
    \hfill
    \begin{minipage}{0.55\textwidth}
        \centering
        \begin{tikzpicture}
            \begin{axis}[
                width=\textwidth,
                height=0.6\textwidth,
                ylabel={Loss},
                xlabel={$k$},
                symbolic x coords={1,2,3,4,5},
                xtick=data,
                legend style={at={(0.5,1.2)}, anchor=north, draw=none, legend columns=-1},
                ymin=0.001, ymax=1,
                ymode=log,
                ytick={1,0.1,0.01,0.001},
                yticklabels={$10^0$,$10^{-1}$,$10^{-2}$,$10^{-3}$},
                enlarge x limits=0.2,
                grid=major,
            ]
            \addplot+[mark=* , color=blue!50, thick, mark options={fill=blue!50}] coordinates {(1,0.015) (2,0.024) (3,0.023) (4,0.019) (5,0.020)};
            \addplot+[mark=square*, color=green!50, thick, mark options={fill=green!50}] coordinates {(1,0.016) (2,0.033) (3,0.021) (4,0.020) (5,0.024)};
            \addplot+[mark=triangle*, color=red!50, thick, mark options={fill=red!50}] coordinates {(1,0.301) (2,0.324) (3,0.299) (4,0.099) (5,0.107)};
            \legend{Val. Perf., Test Perf., Gen. Perf.}
            \end{axis}
        \end{tikzpicture}
    \end{minipage}
    \caption{
            \revised{
                Performance of softmax transformer classifiers for $L_k$ ($k=1$ to $5$).
                \textbf{Validation Performance (Val. Perf.)}: BCEWithLogitsLoss on validation data.
                \textbf{Test Performance (Test Perf.)}: BCEWithLogitsLoss and Accuracy (separated by /) on test data.
                \textbf{Generalization Performance (Gen. Perf.)}: BCEWithLogitsLoss and Accuracy (separated by /) on generalization test set.
                The y-axis uses a logarithmic scale to accommodate the different orders of magnitude in the results.
                }
            }
    \label{fig:table-plot-side-by-side}
\end{figure}

Figure~\ref{fig:table-plot-side-by-side} presents the outcome of our
experiments. The table on the left-hand side demonstrates
the best observed performance on the validation dataset (first
column), a balanced test dataset derived from the same distribution as
the training and validation data (second column). This specifically
implies that this dataset also only includes words of length up to 500. The
final column represents another balanced test dataset encompassing
words from length 501 to 1000, used to potentially unveil some length generalization performance. 
The plot on the right visualizes the same results.

Generally, we observe very high performance with an accuracy of $\ge 0.99$ on
the in-distribution test dataset. Additionally, while the performance on the 
test dataset with longer words decreases, it remains relatively high, with an accuracy of
$\ge 0.94$ in all instances.
Especially, it is to be assumed that with a more extensive experimental setup, this gap in performance will
decrease. Therefore, we infer that our trained encoders perform well
and that length generalization is supported, indicating that the model
can capture the semantics of $L_k$.
\revised{
In Appendix~\ref{app:further_experiments} we report additional results, showing strong performance, with a decrease in performance on longer inputs.
}

\section{Concluding Remarks}\label{sec:conc}
\paragraph{Related Work.} 
Lots of work have been done in recent years on the expressiveness of
transformers for general (not necessarily counting) properties (cf. see 
\citep{transformers_survey}). Counting properties
--- e.g., the languages $\PARITY$ and $\MAJ$ --- have frequently featured in
transformers expressivity research, which highlight their importance. Various 
theoretical transformer
models have been used in the literature employing different assumptions on the
attention mechanisms (hardmax attention vs. softmax attention), positional
encodings, etc. For example, a large proportion of results use hardmax 
attention, which is not used
by practical transformers (which instead use softmax attention). In addition,
some works (e.g.\ \cite{perez,barcelo2023logical}) employ extremely complex
positional encodings with no restrictions. That said, several recent works 
have adopted more practical models.
In particular, the works of \cite{YC24,framework,simulating,yang2025kneedeep} employ
softmax attention transformers and simple classes of positional encodings
(causal masking, local, etc.). Our results also employ a similar model 
(AHAT[U] and SMAT); in fact, we proved that semialgebraic counting properties 
can be captured by transformers without any positional encodings.
\cite{yang2025kneedeep} gave a restriction of softmax attention transformers
with bounded finite precision outside the attention computation, which 
characterizes
C-RASP. Our experimental results seem to suggest this transformer model only 
lower-bounds the expressivity of real-world transformers, which can capture
counting properties beyond C-RASP.

Concerning verification of transformers, we mention the works by \cite{ACY24}
and \cite{succinct}, showing that reasoning about Unique-Hard Attention
Transformers (UHAT) are decidable with complexity EXPSPACE-complete. UHAT is
known to overapproximate what can be captured by softmax transformers with
bounded finite precision \citep{LiC25}. We also mention the recent work
\citep{yang26}, showing that verifying C-RASP is undecidable.

\revised{
    \paragraph{Potential Applications in NLP.} 
%
%
By Weierstrass theorem, polynomials can approximate any continuous function of 
the number of occurrences of tokens. This suggests that transformers can solve
practical NLP tasks that require computation of nonlinear statistics in the word
frequencies. 

Counting properties are tightly
connected to \emph{Vector Space Model 
(VSM)} \citep{vsm,gvsm,SLY19} that has applications in text classification and 
similarity analysis, where the 
standard method has been to employ Support Vector Machines (SVM), together 
with kernel analysis (e.g. using polynomial kernels). Our results imply that 
transformers are expressive enough to perform such tasks. In VSM, 
a document $D$ is a vector $v_D$ indexed by ``terms'' that may occur in $D$. 
That is, $v_D[t]$ is a count on the number of occurrences of 
$t$ in $D$. To compare similarity between two documents $D,D'$, we 
may consider the Euclidean distance between $v_D$ and $v_{D'}$, which requires 
a polynomial. 
Also, there are often challenges including "related terms" (e.g. husband, 
wife, and spouse), which are missed when we only use the aforementioned metric.
Thus, a similarity measure is often learned (see Section 10.2.2 in 
 \citep{kernel-book}, where VSM is used in combination with polynomial kernels). Our results show that transformers 
can solve such a task.
A related task is the problem of determining proximity to a human written text,
as dictated by \cite{zipf35} stating that the frequency 
of the $k$-th most frequent word is proportional to $1/k$ in a natural 
language. As above, we may compare using Euclidean distance 
a document $D$ with a predetermined Zipf-vector. 
This results in a polynomial, and our results show this can be captured by transformers.
}

\paragraph{Future Work.} We mention several open problems. Firstly, can softmax
attention transformers with causal masking capture counting properties beyond
semialgebraic sets? Secondly, our work has identified a gap in the formalization
of the RASP-L conjecture by \cite{framework}. That is, transformers can capture
and efficiently learn semialgebraic counting properties, which are beyond the
language C-RASP. It is open whether the extension of C-RASP with
inequalities over \emph{nonlinear} polynomials can still be captured by softmax
transformers.

\label{beforebibliography}
\newoutputstream{pages}
\openoutputfile{main.pages.ctr}{pages}
\addtostream{pages}{\getpagerefnumber{beforebibliography}}
\closeoutputstream{pages}

\section*{Acknowledgments}
We thank David Chiang, Michael Hahn, Andy Yang, and anonymous reviews for their
feedback. 

\raisebox{-9pt}[0pt][0pt]{\includegraphics[height=.8cm]{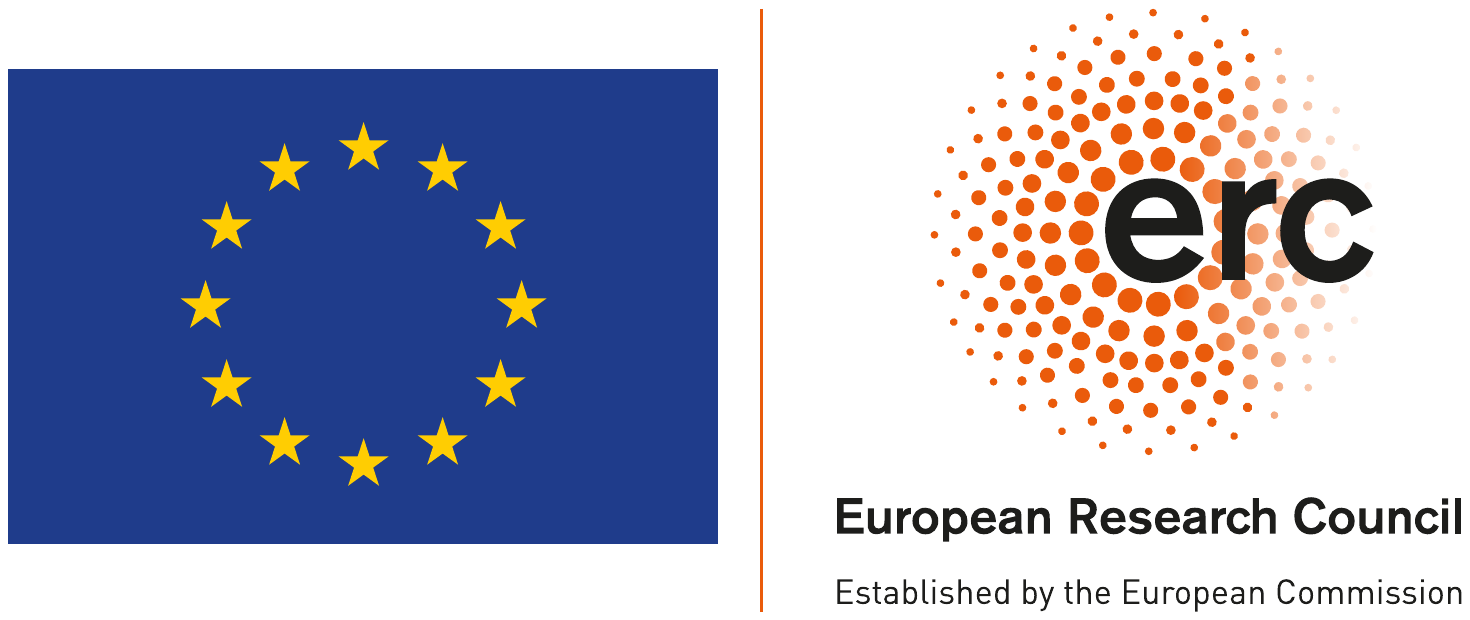}} Marco Sälzer, Chris Köcher, Georg Zetzsche, and Anthony Lin are 
funded by the European Union (ERC, LASD, 101089343 and FINABIS, 101077902). Views and opinions expressed are however those of the authors only and do not necessarily reflect those of the European Union 
or the European Research Council Executive Agency. 
Neither the European Union nor the granting authority can be held responsible for them.

Alexander Kozachinskiy is funded by the National Center for Artificial Intelligence CENIA (FB210017, Basal ANID, and ANID Fondecyt Iniciaci\'{o}n grant 11250060).

\bibliography{refs}

@article{simulating,
  author       = {Andy Yang and
                  Lena Strobl and
                  David Chiang and
                  Dana Angluin},
  title        = {Simulating Hard Attention Using Soft Attention},
  journal      = {CoRR},
  volume       = {abs/2412.09925},
  year         = {2024},
  url          = {https://doi.org/10.48550/arXiv.2412.09925},
  doi          = {10.48550/ARXIV.2412.09925},
  eprinttype    = {arXiv},
  eprint       = {2412.09925},
  bibsource    = {dblp computer science bibliography, https://dblp.org}
}

@inproceedings{BAG20,
  author       = {Satwik Bhattamishra and
                  Kabir Ahuja and
                  Navin Goyal},
  editor       = {Bonnie Webber and
                  Trevor Cohn and
                  Yulan He and
                  Yang Liu},
  title        = {On the Ability and Limitations of Transformers to Recognize Formal
                  Languages},
  booktitle    = {Proceedings of the 2020 Conference on Empirical Methods in Natural
                  Language Processing, {EMNLP} 2020, Online, November 16-20, 2020},
  pages        = {7096--7116},
  publisher    = {Association for Computational Linguistics},
  year         = {2020},
  url          = {https://doi.org/10.18653/v1/2020.emnlp-main.576},
  doi          = {10.18653/V1/2020.EMNLP-MAIN.576},
  bibsource    = {dblp computer science bibliography, https://dblp.org}
}

@inproceedings{chomsky,
  author       = {Gr{\'{e}}goire Del{\'{e}}tang and
                  Anian Ruoss and
                  Jordi Grau{-}Moya and
                  Tim Genewein and
                  Li Kevin Wenliang and
                  Elliot Catt and
                  Chris Cundy and
                  Marcus Hutter and
                  Shane Legg and
                  Joel Veness and
                  Pedro A. Ortega},
  title        = {Neural Networks and the Chomsky Hierarchy},
  booktitle    = {The Eleventh International Conference on Learning Representations,
                  {ICLR} 2023, Kigali, Rwanda, May 1-5, 2023},
  publisher    = {OpenReview.net},
  year         = {2023},
  url          = {https://openreview.net/forum?id=WbxHAzkeQcn},
  bibsource    = {dblp computer science bibliography, https://dblp.org}
}

@article{hahn20,
  author       = {Michael Hahn},
  title        = {Theoretical Limitations of Self-Attention in Neural Sequence Models},
  journal      = {Trans. Assoc. Comput. Linguistics},
  volume       = {8},
  pages        = {156--171},
  year         = {2020},
  url          = {https://doi.org/10.1162/tacl\_a\_00306},
  doi          = {10.1162/TACL\_A\_00306},
  bibsource    = {dblp computer science bibliography, https://dblp.org}
}

@article{perez,
  author       = {Jorge P{\'{e}}rez and
                  Pablo Barcel{\'{o}} and
                  Javier Marinkovic},
  title        = {Attention is Turing-Complete},
  journal      = {J. Mach. Learn. Res.},
  volume       = {22},
  pages        = {75:1--75:35},
  year         = {2021},
  url          = {https://jmlr.org/papers/v22/20-302.html},
  bibsource    = {dblp computer science bibliography, https://dblp.org}
}

@inproceedings{CC22,
  author       = {David Chiang and
                  Peter Cholak},
  editor       = {Smaranda Muresan and
                  Preslav Nakov and
                  Aline Villavicencio},
  title        = {Overcoming a Theoretical Limitation of Self-Attention},
  booktitle    = {Proceedings of the 60th Annual Meeting of the Association for Computational
                  Linguistics (Volume 1: Long Papers), {ACL} 2022, Dublin, Ireland,
                  May 22-27, 2022},
  pages        = {7654--7664},
  publisher    = {Association for Computational Linguistics},
  year         = {2022},
  url          = {https://doi.org/10.18653/v1/2022.acl-long.527},
  doi          = {10.18653/V1/2022.ACL-LONG.527},
  bibsource    = {dblp computer science bibliography, https://dblp.org}
}

@inproceedings{sensitivity,
  author       = {Michael Hahn and
                  Mark Rofin},
  editor       = {Lun{-}Wei Ku and
                  Andre Martins and
                  Vivek Srikumar},
  title        = {Why are Sensitive Functions Hard for Transformers?},
  booktitle    = {Proceedings of the 62nd Annual Meeting of the Association for Computational
                  Linguistics (Volume 1: Long Papers), {ACL} 2024, Bangkok, Thailand,
                  August 11-16, 2024},
  pages        = {14973--15008},
  publisher    = {Association for Computational Linguistics},
  year         = {2024},
  url          = {https://doi.org/10.18653/v1/2024.acl-long.800},
  doi          = {10.18653/V1/2024.ACL-LONG.800},
  bibsource    = {dblp computer science bibliography, https://dblp.org}
}

@article{transformers_survey,
  author       = {Lena Strobl and
                  William Merrill and
                  Gail Weiss and
                  David Chiang and
                  Dana Angluin},
  title        = {What Formal Languages Can Transformers Express? {A} Survey},
  journal      = {Trans. Assoc. Comput. Linguistics},
  volume       = {12},
  pages        = {543--561},
  year         = {2024},
  url          = {https://doi.org/10.1162/tacl\_a\_00663},
  doi          = {10.1162/TACL\_A\_00663},
  bibsource    = {dblp computer science bibliography, https://dblp.org}
}

@inproceedings{Vaswani17,
  author       = {Ashish Vaswani and
                  Noam Shazeer and
                  Niki Parmar and
                  Jakob Uszkoreit and
                  Llion Jones and
                  Aidan N. Gomez and
                  Lukasz Kaiser and
                  Illia Polosukhin},
  editor       = {Isabelle Guyon and
                  Ulrike von Luxburg and
                  Samy Bengio and
                  Hanna M. Wallach and
                  Rob Fergus and
                  S. V. N. Vishwanathan and
                  Roman Garnett},
  title        = {Attention is All you Need},
  booktitle    = {Advances in Neural Information Processing Systems 30: Annual Conference
                  on Neural Information Processing Systems 2017, December 4-9, 2017,
                  Long Beach, CA, {USA}},
  pages        = {5998--6008},
  year         = {2017},
  url          = {https://proceedings.neurips.cc/paper/2017/hash/3f5ee243547dee91fbd053c1c4a845aa-Abstract.html},
  bibsource    = {dblp computer science bibliography, https://dblp.org}
}

@inproceedings{ACY24,
  author = {Yang, Andy and Chiang, David and Angluin, Dana},
  booktitle = {Advances in Neural Information Processing Systems},
  editor = {A. Globerson and L. Mackey and D. Belgrave and A. Fan and U. Paquet and J. Tomczak and C. Zhang},
  pages = {10202--10235},
  publisher = {Curran Associates, Inc.},
  title = {Masked Hard-Attention Transformers Recognize Exactly the Star-Free Languages},
  url = {https://proceedings.neurips.cc/paper_files/paper/2024/file/13d7f172259b11b230cc5da8768abc5f-Paper-Conference.pdf},
  volume = {37},
  year = {2024}
}

@inproceedings{barcelo2023logical,
  author       = {Pablo Barcel{\'{o}} and
                  Alexander Kozachinskiy and
                  Anthony Widjaja Lin and
                  Vladimir V. Podolskii},
  title        = {Logical Languages Accepted by Transformer Encoders with Hard Attention},
  booktitle    = {{ICLR}},
  publisher    = {OpenReview.net},
  year         = {2024}
}

@inproceedings{SAL25,
  author       = {Marco S{\"{a}}lzer and
                  Eric Alsmann and
                  Martin Lange},
  title        = {Transformer Encoder Satisfiability: Complexity and Impact on Formal
                  Reasoning},
  booktitle    = {The Thirteenth International Conference on Learning Representations,
                  {ICLR} 2025, Singapore, April 24-28, 2025},
  publisher    = {OpenReview.net},
  year         = {2025},
  url          = {https://openreview.net/forum?id=VVO3ApdMUE}
}

@article{HAF22,
  author       = {Yiding Hao and
                  Dana Angluin and
                  Robert Frank},
  title        = {Formal Language Recognition by Hard Attention Transformers: Perspectives
                  from Circuit Complexity},
  journal      = {Trans. Assoc. Comput. Linguistics},
  volume       = {10},
  pages        = {800--810},
  year         = {2022},
  url          = {https://doi.org/10.1162/tacl\_a\_00490},
  doi          = {10.1162/TACL\_A\_00490},
  bibsource    = {dblp computer science bibliography, https://dblp.org}
}

@Book{Mat93,
  title =        "Hilbert's Tenth Problem",
  author =       "Yuri V. Matiyasevich",
  publisher =    "MIT Press",
  address =      "Cambridge, Massachusetts",
  year =         "1993",
}

@book{sipser-book,
  address = {Boston, MA},
  author = {Sipser, Michael},
  edition = {Third},
  isbn = {113318779X},
  publisher = {Course Technology},
  refid = {814441519},
  title = {Introduction to the Theory of Computation},
  year = 2013
}

@article{DBLP:journals/siglog/Haase18,
  author       = {Christoph Haase},
  title        = {A survival guide to presburger arithmetic},
  journal      = {{ACM} {SIGLOG} News},
  volume       = {5},
  number       = {3},
  pages        = {67--82},
  year         = {2018},
  doi          = {10.1145/3242953.3242964},
}

@inproceedings{DBLP:conf/fsttcs/000124,
  author       = {Dmitry Chistikov},
  editor       = {Siddharth Barman and
                  Slawomir Lasota},
  title        = {An Introduction to the Theory of Linear Integer Arithmetic (Invited
                  Paper)},
  booktitle    = {44th {IARCS} Annual Conference on Foundations of Software Technology
                  and Theoretical Computer Science, {FSTTCS} 2024, December 16-18, 2024,
                  Gandhinagar, Gujarat, India},
  series       = {LIPIcs},
  volume       = {323},
  pages        = {1:1--1:36},
  publisher    = {Schloss Dagstuhl - Leibniz-Zentrum f{\"{u}}r Informatik},
  year         = {2024},
  doi          = {10.4230/LIPICS.FSTTCS.2024.1},
}

@inproceedings{framework,
  author       = {Xinting Huang and
                  Andy Yang and
                  Satwik Bhattamishra and
                  Yash Raj Sarrof and
                  Andreas Krebs and
                  Hattie Zhou and
                  Preetum Nakkiran and
                  Michael Hahn},
  title        = {A Formal Framework for Understanding Length Generalization in Transformers},
  booktitle    = {The Thirteenth International Conference on Learning Representations,
                  {ICLR} 2025, Singapore, April 24-28, 2025},
  publisher    = {OpenReview.net},
  year         = {2025},
  url          = {https://openreview.net/forum?id=U49N5V51rU},
  bibsource    = {dblp computer science bibliography, https://dblp.org}
}

@inproceedings{anil22,
  author       = {Cem Anil and
                  Yuhuai Wu and
                  Anders Andreassen and
                  Aitor Lewkowycz and
                  Vedant Misra and
                  Vinay V. Ramasesh and
                  Ambrose Slone and
                  Guy Gur{-}Ari and
                  Ethan Dyer and
                  Behnam Neyshabur},
  editor       = {Sanmi Koyejo and
                  S. Mohamed and
                  A. Agarwal and
                  Danielle Belgrave and
                  K. Cho and
                  A. Oh},
  title        = {Exploring Length Generalization in Large Language Models},
  booktitle    = {Advances in Neural Information Processing Systems 35: Annual Conference
                  on Neural Information Processing Systems 2022, NeurIPS 2022, New Orleans,
                  LA, USA, November 28 - December 9, 2022},
  year         = {2022},
  url          = {http://papers.nips.cc/paper\_files/paper/2022/hash/fb7451e43f9c1c35b774bcfad7a5714b-Abstract-Conference.html},
  bibsource    = {dblp computer science bibliography, https://dblp.org}
}

@article{YC24,
  author       = {Andy Yang and
                  David Chiang},
  title        = {Counting Like Transformers: Compiling Temporal Counting Logic Into
                  Softmax Transformers},
  journal      = {CoRR},
  volume       = {abs/2404.04393},
  year         = {2024}
}

@article{vsm,
  author       = {Gerard Salton and
                  Anita Wong and
                  Chung{-}Shu Yang},
  title        = {A Vector Space Model for Automatic Indexing},
  journal      = {Commun. {ACM}},
  volume       = {18},
  number       = {11},
  pages        = {613--620},
  year         = {1975},
  url          = {https://doi.org/10.1145/361219.361220},
  doi          = {10.1145/361219.361220},
  bibsource    = {dblp computer science bibliography, https://dblp.org}
}

@book{kernel-book,
  asin = {0521813972},
  author = {Shawe-Taylor, John and Cristianini, Nello},
  description = {Amazon.com: Kernel Methods for Pattern Analysis (9780521813976): John Shawe-Taylor, Nello Cristianini: Books},
  dewey = {006.31},
  ean = {9780521813976},
  edition = {illustrated edition},
  isbn = {0521813972},
  publisher = {Cambridge University Press},
  title = {Kernel Methods for Pattern Analysis},
  url = {http://www.amazon.com/Kernel-Methods-Pattern-Analysis-Shawe-Taylor/dp/0521813972},
  year = 2004
}

@inproceedings{gvsm,
  author       = {S. K. Michael Wong and
                  Wojciech Ziarko and
                  P. C. N. Wong},
  editor       = {Jean Tague},
  title        = {Generalized Vector Space Model in Information Retrieval},
  booktitle    = {Proceedings of the 8th annual international {ACM} {SIGIR} conference
                  on Research and development in information retrieval, Montr{\'{e}}al,
                  Qu{\'{e}}bec, Canada, June 5-7, 1985},
  pages        = {18--25},
  publisher    = {{ACM}},
  year         = {1985},
  url          = {https://doi.org/10.1145/253495.253506},
  doi          = {10.1145/253495.253506},
  bibsource    = {dblp computer science bibliography, https://dblp.org}
}

@article{bow,
  author = {Harris, Zellig},
  doi = {10.1007/978-94-009-8467-7_1},
  journal = {Word},
  number = {2-3},
  pages = {146--162},
  publisher = {Taylor \& Francis},
  title = {Distributional structure},
  url = {https://link.springer.com/chapter/10.1007/978-94-009-8467-7_1},
  volume = 10,
  year = 1954
}

@inproceedings{raspl,
  author       = {Hattie Zhou and
                  Arwen Bradley and
                  Etai Littwin and
                  Noam Razin and
                  Omid Saremi and
                  Joshua M. Susskind and
                  Samy Bengio and
                  Preetum Nakkiran},
  title        = {What Algorithms can Transformers Learn? {A} Study in Length Generalization},
  booktitle    = {The Twelfth International Conference on Learning Representations,
                  {ICLR} 2024, Vienna, Austria, May 7-11, 2024},
  publisher    = {OpenReview.net},
  year         = {2024},
  url          = {https://openreview.net/forum?id=AssIuHnmHX},
  bibsource    = {dblp computer science bibliography, https://dblp.org}
}

@article{MerrillS23,
  author       = {William Merrill and
                  Ashish Sabharwal},
  title        = {The Parallelism Tradeoff: Limitations of Log-Precision Transformers},
  journal      = {Trans. Assoc. Comput. Linguistics},
  volume       = {11},
  pages        = {531--545},
  year         = {2023},
  url          = {https://doi.org/10.1162/tacl\_a\_00562},
  doi          = {10.1162/TACL\_A\_00562}
}

@inproceedings{LiC25,
title={Characterizing the Expressivity of Fixed-Precision Transformer Language Models},
author={Jiaoda Li and Ryan Cotterell},
booktitle={The Thirty-ninth Annual Conference on Neural Information Processing Systems},
year={2025},
url={https://openreview.net/forum?id=29LwAgLFpj}
}

@book{zipf35,
  title     = {The Psychobiology of Language: An Introduction to Dynamic Philology},
  author    = {Zipf, George Kingsley},
  year      = {1935},
  publisher = {Houghton Mifflin},
  address   = {Boston, MA}
}

@inproceedings{MS23-nips,
  author       = {William Merrill and
                  Ashish Sabharwal},
  editor       = {Alice Oh and
                  Tristan Naumann and
                  Amir Globerson and
                  Kate Saenko and
                  Moritz Hardt and
                  Sergey Levine},
  title        = {A Logic for Expressing Log-Precision Transformers},
  booktitle    = {Advances in Neural Information Processing Systems 36: Annual Conference
                  on Neural Information Processing Systems 2023, NeurIPS 2023, New Orleans,
                  LA, USA, December 10 - 16, 2023},
  year         = {2023},
  url          = {http://papers.nips.cc/paper\_files/paper/2023/hash/a48e5877c7bf86a513950ab23b360498-Abstract-Conference.html},
  bibsource    = {dblp computer science bibliography, https://dblp.org}
}

@inproceedings{SLY19,
  author       = {Omid Shahmirzadi and
                  Adam Lugowski and
                  Kenneth Younge},
  editor       = {M. Arif Wani and
                  Taghi M. Khoshgoftaar and
                  Dingding Wang and
                  Huanjing Wang and
                  Naeem Seliya},
  title        = {Text Similarity in Vector Space Models: {A} Comparative Study},
  booktitle    = {18th {IEEE} International Conference On Machine Learning And Applications,
                  {ICMLA} 2019, Boca Raton, FL, USA, December 16-19, 2019},
  pages        = {659--666},
  publisher    = {{IEEE}},
  year         = {2019},
  url          = {https://doi.org/10.1109/ICMLA.2019.00120},
  doi          = {10.1109/ICMLA.2019.00120},
  bibsource    = {dblp computer science bibliography, https://dblp.org}
}

@inproceedings{yang2025kneedeep,
title={Knee-Deep in C-{RASP}: A Transformer Depth Hierarchy},
author={Andy Yang and Micha{\"e}l Cadilhac and David Chiang},
booktitle={The Thirty-ninth Annual Conference on Neural Information Processing Systems},
year={2025},
url={https://openreview.net/forum?id=jPduiyxyfw}
}

@inproceedings{succinct,
title={Transformers are Inherently Succinct},
author={Pascal Bergstr{\"a}{\ss}er and Ryan Cotterell and Anthony Widjaja Lin},
booktitle={The Fourteenth International Conference on Learning Representations},
year={2026},
url={https://openreview.net/forum?id=Yxz92UuPLQ}
}

@misc{yang26,
    title = {Length Generalization Bounds for Transformers},
    author = {Andy Yang and Pascal Bergsträßer and Georg Zetzsche and David
        Chiang and Anthony Lin},
    year = 2026,
    note = {Under submission (preprint: \url{https://zenodo.org/records/18800700})},
    doi = "10.5281/zenodo.18800700"
}
\bibliographystyle{iclr2026_conference}

\appendix

\section{Translating semialgebraic sets to $\AHATem$}\label{app:semialg-to-ahat}
\subsection{Fine-grained analysis of polynomial degree vs. depth}
In this subsection, we show the inclusion $\SemiAlg\subseteq\AHATem[U]$. In fact, we show a stronger statement (\cref{semialg-parametric}), which requires some notation.
By $\SemiAlg[\leq\ell]$ we denote the restriction of the class $\SemiAlg$ to the semi-algebraic languages $L\subseteq\Sigma^*$ such that the underlying semi-algebraic set $S\subseteq\N^m$ is a Boolean combination of sets $S_p$ where $p\in\Z[X_1,\ldots,X_m]$ are polynomials of degree $\leq\ell$. In particular, we have $\SemiAlg[\leq1]=\QFPA$. Our construction for $\SemiAlg\subseteq\AHATem[U]$ actually shows the following:
\begin{proposition}\label{semialg-parametric}
  For each $\ell>0$ we have $\SemiAlg[\leq\ell]\subseteq\AHATem[\leq\ell,U]$.
\end{proposition}

For showing \cref{semialg-parametric}, we need some more technical definitions. Let $T$ be an AHAT with input embedding $\emb\colon\Sigma\to\Q^{d_1}$ and layers $\lambda_1\colon(\Q^{d_1})^*\to(\Q^{d_2})^*,\ldots,\lambda_\ell\colon(\Q^{d_\ell})^*\to(\Q^{d_{\ell+1}})^*$. We define the function $f_T\colon\Sigma^+\to\Q$ as follows: for a word $w=a_1a_2\ldots a_n\in\Sigma^+$, if $\lambda_1\circ\cdots\circ\lambda_\ell(\emb(a_1),\ldots,\emb(a_n))=(\by_1,\ldots,\by_n)$, then $f_T(w)=\by_n[1]$. In other words, we have $f_T(w)>0$ iff $T(w)=1$.

\begin{proposition}\label{app:polynomial-inequality}
  For every polynomial $p\in\Z[X_1,\ldots,X_m]$ of degree $\ell$, the language $L_{p>0}=\{w\in\{\lta_1,\ldots,\lta_m\}^*\mid p(\Psi(w))>0\}$ belongs to $\AHATem[\leq\ell,U]$.
\end{proposition}

To show \cref{polynomial-inequality}, we will use polynomials that are \emph{homogeneous},
meaning all monomials have the same degree. Note that given an arbitrary
polynomial $p\in\Z[X_1,\ldots,X_m]$ of degree $\ell$, we can consider the
polynomial $q\in\Z[X_0,\ldots,X_m]$ with
$q=X_0^dp(\tfrac{X_1}{X_0},\ldots,\tfrac{X_m}{X_0})$, which is homogeneous. It
has the property that $p(x_1,\ldots,x_m)>0$ if and only if
$q(1,x_1,\ldots,x_m)>0$. Therefore, from now on, we assume that we have a homogeneous polynomial $q\in\Z[X_0,\ldots,X_m]$ and want to construct an AHAT for the language $K_q=\{w\in\{\lta_1,\ldots,\lta_m\}^* \mid \text{$q(1,\bx)>0$ for $\bx=\Psi(w)$}\}$.

To simplify notation, we denote the end marker $\$$ by $\lta_0$.
Thus, the input will be a string $w\in\{\lta_0,\ldots,\lta_m\}^+$
that contains $\lta_0$ exactly once, at the end.  Since $|w|_{\lta_0}=1$ is
satisfied automatically, our AHAT only has to check that $q(x_0,\ldots,x_m)>0$,
where $x_i=|w|_{\lta_i}$.
The input encoding is the map $\{\lta_0,\ldots,\lta_m\}^*\to \Q^m$ with
$\lta_i\mapsto\be_i$, where $\be_i\in\Q^m$ is the $i$-th unit vector.

In a first lemma we show that each monomial of $q$ can be computed by a NoPE-AHAT with $\ell$ uniform attention layers.

\begin{lemma}\label{lem:monomialToAHAT}
  For every monomial $r\in\Z[X_0,X_1,\ldots,X_m]$ of degree $\ell$, there is a NoPE-AHAT $T$ with $\ell$ uniform attention layers such that
  \[f_T(w)=\frac{r(\Psi(w))}{|w|^\ell}\]
  for each word $w\in\Sigma^*$. In particular, we have $f_T(w\$)>0$ if and only if $r(\Psi(w))>0$.
\end{lemma}
\begin{proof}
  We use the word embedding $\emb\colon\Sigma\to\Q^{m+1}$ with $\emb(\lta_i)=\be_i$ for each $i\in[0,m]$.
  
  \paragraph{Step I: Compute frequencies}
  Our AHAT first uses an attention layer to compute $m+1$ new components, where
  $i$-th component holds $\tfrac{x_i}{n+1}$, where $n+1$ is the length of the
  input (including the end marker). This is easily done by attending to all
  positions and computing the averages of the first $m+1$ components.
  To simplify notation, we will index vectors starting with index $0$.

  \paragraph{Step II: Multiplication gadgets}
  Second, we have a sequence of gadgets (each consisting of one uniform attention layer and one ReLU layer). Each gadget introduces a new component, and does not change the existing components.
  Between gadget executions, the following additional invariants are upheld:
  \begin{enumerate*}[label=(\roman*)]
    \item Overall, a gadget does not change existing components: it introduces one new component.
    \item The components $\{0,\ldots,m\}$ are called the \emph{initial} components.
    \item All other components are \emph{uniform}, i.e.\ they are the same across all positions. 
    \item The uniform components carry values in $[0,1]$.
  \end{enumerate*}
  Thus, we will call components $0,\ldots,m$ the \emph{initial} components; and we call components $>m$ the \emph{uniform} components.
  
  Our gadgets do the following. Suppose we have already produced $k$ additional components. For each initial component $i\in[0,m]$ and uniform component $j\in[m+1,m+1+k]$, gadget $\omult(k,i,j)$, which introduces a new component, will carry the value $\frac{x_i\cdot y_j}{n+1},$
  where $y_j$ is the value in component $j$ of all vectors. Recall that we use $x_i$ to denote the number of $\lta_i$ 
  occurrences in the input for $i\in[0,m]$.
  
  We implement the gadget $\omult(k,i,j)$ using some ReLU layers and an
  attention layer. Suppose that before, we have the vector
  $\bu_p\in\Q^{m+1+k}$ in position $p$.  First, using ReLU layers, we
  introduce a new component that in position $p$ has the value $\bu_p[i]\cdot
  \bu_p[j]$. This can be achieved since $\bu_p[i]$ is in $\{0,1\}$ and
  $\bu_p[j]\in[0,1]$: Notice that $\bu_p[i]\cdot\bu_p[j]=\ReLU(\bu_p[j]-(1-\bu_p[i]))$.
  Indeed, if $\bu_p[i]=1$, then this evaluates to $\bu_p[j]$; if
  $\bu_p[i]=0$, then we get $\ReLU(\bu_p[j]-1)=0$.
  We then use uniform attention to compute the average of this new
  $\bu_p[i]\cdot\bu_p[j]$-component across all vectors. Since there are $n+1$
  vectors, exactly $x_i$ of them have $\bu_p[i]=1$, and also $\bu_p[j]=y_j$, we get the desired
  $\tfrac{x_i\cdot y_j}{n+1}$.
  
  \paragraph{Step III: Computing the monomial}
  We now use our gadgets to compute the value of the monomial. Let $r(X_0,\ldots,X_m)=\alpha\cdot X_{i_1}\cdots X_{i_\ell}$. We use $\ell-1$ gadgets to compute
  $x_{i_1}\cdots x_{i_\ell}/(n+1)^\ell$: The frequency computation in the beginning
  yields $x_{i_1}/(n+1)$, and then we use gadgets to compute
  $x_{i_1}x_{i_2}/(n+1)^2$, $x_{i_1}x_{i_2}x_{i_3}/(n+1)^3$, etc.\ until
  $x_{i_1}\cdots x_{i_\ell}/(n+1)^\ell$. Finally, we use a ReLU layer to multiply
  $x_{i_1}\cdots x_{i_\ell}/(n+1)^\ell$ with $\alpha$. Thus, we have computed $r(x_0,\ldots,x_m)/(n+1)^\ell$.
\end{proof}

\subsection{Combining $\AHATem[U]$ without additional layers}\label{app:union-intersection}
The following lemma states that two NoPE-AHAT with only uniform attention layers can be parallelized resulting in a NoPE-AHAT with the same number of uniform layers. Their outputs can also be combined via a ReLU neural network. In particular, $\AHATem[\le \ell,U]$ is closed under union and intersection.

\begin{lemma}\label{lem:parallel}
  Let $T_1,T_2$ be two NoPE-AHAT with $\ell$ uniform attention layers and let $\mathcal{N}$ be a ReLU neural network computing a function $\mathcal{N}\colon\Q^2\to\Q$. Then there is a NoPE-AHAT $T_\mathcal{N}$ with $\ell$ uniform attention layers computing $f_{T_{\mathcal{N}}}(w\$)=\mathcal{N}(f_{T_1}(w\$),f_{T_2}(w\$))$.
\end{lemma}
\begin{proof}
  The idea of $T_{\mathcal{N}}$ is, that it concatenates the components from $T_1$ with those of $T_2$ and keeps the sets of components always disjoint. By uniformity we are able to apply the attention layers of $T_1$ and $T_2$ in parallel. In the last attention layer we can simply apply $\mathcal{N}$ to the first components of $T_1$ and $T_2$.
  
  By $\emb_i\colon\Sigma\to\Q^{d_{1,i}}$ we denote the word embedding of $T_i$. From this we construct a new word embedding $\emb\colon\Sigma\to\Q^{d_{1,1}+d_{1,2}}$ with $\emb(\lta_j)=(\emb_1(\lta_j),\emb_2(\lta_j))$ for each $j\in[0,m]$.
  
  Now, let $\lambda_{k,i}\colon\Q^{d_{k,i}}\to\Q^{d_{k+1,i}}$ be the $k$th layer of $T_i$ for $1\leq k\leq\ell$. By $K_i,Q_i,V_i$, and $\mathcal{N}_i$ we denote the parameters of $\lambda_{k,i}$. Since $\lambda_{k,i}$ is uniform, the key and query maps $K_i$ and $Q_i$ are constantly mapping to zero. We now construct a uniform layer $\lambda_k\colon \Q^{d_{k,1}+d_{k,2}}\to\Q^{d_{k+1,1}+d_{k+1},2}$ composed of $\lambda_{k,1}$ and $\lambda_{k,2}$: the key and query maps $K$ and $Q$ still map to zero. If $V_i(\bx_i)=A_i\bx+\bbb_i$ then we define the new value map $V$ by
  \[V\binom{\bx_1}{\bx_2}=\begin{pmatrix}
    A_1 & \bzero\\
    \bzero & A_2
  \end{pmatrix}\binom{\bx_1}{\bx_2}+\binom{\bbb_1}{\bbb_2}=\binom{A_1\bx_1+\bbb_1}{A_2\bx_2+\bbb_2}=\binom{V_1(\bx_1)}{V_2(\bx_2)}\,.\]
  By this definition we obtain that the attention vectors $\ba_j$ in $\lambda_k$ are the concatenation of the attention vectors $\ba_{j,1}$ and $\ba_{j,2}$ in $\lambda_{k,1}$ resp.\ $\lambda_{k,2}$.
  Similarly, we build the composition of $\mathcal{N}_1$ and $\mathcal{N}_2$ resulting in an FFN computing $\binom{\mathcal{N}_1(\bx_{j,1},\ba_{j,1})}{\mathcal{N}_2(\bx_{j,2},\ba_{j,2})}$.
  
  Finally, in the last layer, we add the FFN $\mathcal{N}'$ that takes the first components of the output of $\mathcal{N}_i(\bx_{j,i},\ba_{j,i})$ and simulates $\mathcal{N}$ on these two numbers.
\end{proof}

\begin{figure}
  \begin{center}
    \newcommand{\seq}[2][]{
  \draw[shift={#2},#1] (0,0) rectangle (3.6,0.6);
  \draw[step=0.6,shift={#2}] (0,0) grid (3.6,0.6);
}

\newcommand{\pwl}[3][]{
  \draw[->,shift={#2}] (1.8,2.1) -- (1.8,1.6);
  \draw[shift={#2},#1] (1,0) rectangle (2.6,1.6);
  \node[shift={#2}] at (1.8,0.8) {#3};
  \draw[->,shift={#2}] (1.8,0) -- (1.8,-0.5);
}

\newcommand{\pwlp}[5]{
  \draw[->,shift={#1}] (1.8,2.1) -- (1.8,1.85) -- (0.8,1.85) -- (0.8,1.6);
  \draw[shift={#1},#2] (0,0) rectangle (1.6,1.6);
  \node[shift={#1}] at (0.8,0.8) {#3};
  \draw[->,shift={#1}] (1.6,0.8) -- (2,0.8);
  \draw[shift={#1},#4] (2,0) rectangle (3.6,1.6);
  \node[shift={#1}] at (2.8,0.8) {#5};
  \draw[->,shift={#1}] (2.8,0) -- (2.8,-0.25) -- (1.8,-0.25) -- (1.8,-0.5);
}

\newcommand{\al}[2][]{
  \draw[shift={#2},#1] (0,0) rectangle (3.6,3.6);
  \draw[step=0.6,shift={#2}] (0,0) grid (3.6,3.6);
  \draw[shift={#2},#1] (-0.8,0) rectangle (-0.2,3.6);
  \draw[step=0.6,xshift=-0.8cm,shift={#2}] (0,0) grid (0.6,3.6);
  \node[shift={#2},fill=white,inner sep=2pt,#1] at (1.8,1.8) {\Huge$0$};
}
\usetikzlibrary{patterns,patterns.meta,arrows.meta}
\begin{tikzpicture}[>={Latex[length=1mm]},scale=0.5,every node/.append style={scale=0.5}]
  \node[anchor=east] at (0,1.8) {\Large$T_1\colon$};
  \seq{(0,1.5)}
  \draw[->] (1.8,1.5) -- node[right] {\Large$\emb_1$} (1.8,0.6);
  \seq{(0,0)}
  \al[fill=red!30]{(0,-3.8)}
  \node[anchor=west] at (3.6,-2) {\Large$\rightsquigarrow\vec{a}_1$};
  \pwl[fill=red!30]{(0,-5.9)}{\Large$\mathcal{N}_1$}
  \seq{(0,-7)}
  
  \node[anchor=east] at (6,1.8) {\Large$T_2\colon$};
  \seq{(6,1.5)}
  \draw[->] (7.8,1.5) -- node[right] {\Large$\emb_2$} (7.8,0.6);
  \seq{(6,0)}
  \al[fill=blue!30]{(6,-3.8)}
  \node[anchor=west] at (9.6,-2) {\Large$\rightsquigarrow\vec{a}_2$};
  \pwl[fill=blue!30]{(6,-5.9)}{\Large$\mathcal{N}_2$}
  \seq{(6,-7)}
  
  \node at (12,-3) {\Large$\Rightarrow$};
  
  \node[anchor=east] at (14,1.8) {\Large$T\colon$};
  \seq{(14,1.5)}
  \draw[->] (15.8,1.5) -- node[right] {\Large$\binom{\emb_1}{\emb_2}$} (15.8,0.6);
  \seq{(14,0)}
  \al[fill=red!30,postaction={pattern={Lines[distance=0.3cm,line width=0.15cm,angle=0,yshift=-0.05cm]},pattern color=blue!30}]{(14,-3.8)}
  \node[anchor=west] at (17.6,-2) {\Large$\rightsquigarrow\binom{\vec{a}_1}{\vec{a}_2}$};
  \pwlp{(14,-5.9)}{fill=red!30}{\Large$\binom{\mathcal{N}_1}{1}$}{fill=blue!30}{\Large$\binom{1}{\mathcal{N}_2}$}
  \seq{(14,-7)}
\end{tikzpicture}
  \end{center}
  \caption{Visualization of the proof of \cref{lem:parallel}.}
\end{figure}

Recall that from a polynomial $p\in\Z[X_1,\ldots,X_m]$ we constructed a homogeneous polynomial $q\in\Z[X_0,X_1,\ldots,X_m]$ such that $p(\bx)>0$ if and only if $q(1,\bx)>0$ holds for all vectors $\bx\in\Q^m$. Let $r_1,\ldots,r_k\in\Z[X_0,X_1,\ldots,X_m]$ be the monomials in $q$. Since $q$ is homogeneous, all monomials have the same degree $\ell$. \cref{lem:monomialToAHAT} yields NoPE-AHATs $T_1,\ldots,T_k$ that are computing the monomials $r_i$. Each of these AHATs has exactly $\ell$ uniform attention layers. Finally, we can apply \cref{lem:parallel} to construct a NoPE-AHAT $T$ with $\ell$ uniform layers computing $f_T(w\$)=\frac{q(\Psi(w\$))}{|w\$|^\ell}$ (since addition is an affine map). Then $T$ accepts $w$ iff $\frac{q(\Psi(w\$))}{|w\$|^\ell}>0$ iff $q(\Psi(w\$))>0$ iff $p(\Psi(w))>0$. In other words, $T$ accepts the language $L_{p>0}$.

\subsection{Inexpressibility of $\PARITY$}\label{app:parity}

\revised{
\begin{proof}[Proof of \cref{result-parity}]
By \cref{main-result-semialgebraic}, it 
suffices to show that $\PARITY$ is
	not semi-algebraic. Suppose it is. Then there is a disjunction of
	conjunctions of polynomial inequalities that characterizes $\PARITY$.
	The polynomials are over $\Z[X,Y]$, where $X$ is the variable for
	$\lta$'s and $Y$ is the variable for $\ltb$'s. By plugging in $Y=0$, we
	conclude that the set of even numbers is semi-algebraic. Hence, there is
	a disjunction $\bigvee_{i=1}^n\bigwedge_{j=1}^m p_{i,j}(X)>0$ of
	conjunctions that is satisfied exactly for the even numbers. This
	implies that for some $i$, there are infinitely many even numbers $k$
	such that $\bigwedge_{j=1}^m p_{i,j}(k)>0$. Therefore, for every
	$j\in[1,m]$, the leading coefficient of $p_{i,j}$ must be positive. But
	then, $\bigwedge_{j=1}^m p_{i,j}(k)>0$ must hold for all sufficiently
	large $k$, not just the even ones, a contradiction.
\end{proof}
}



\section{Parametric analysis}\label{app:parametric}
In this section, we study how the expressive power of NoPE-AHAT[U] and SMAT
depends on the number of attention layers. In particular, we show that
\cref{main-result-universality,main-result-undecidability} hold already in the
case of two layers.  The main insight of this proof is that the number of
layers needed to express a semialgebraic set depends on the degrees of the
involved polynomials (see \cref{semialg-parametric}): 
Note that our sketch of an $\AHATem$ for $L_{p>0}$ in \cref{sec:semialgebraic}
directly yields a $\AHATem$ with $\ell$ layers, where $\ell$ is the degree of
$p$. For \cref{semialg-parametric}, one then has to show that Boolean
combinations of such sets can be expressed without growing the number of
attention layers. See \cref{app:semialg-to-ahat} for details.

\paragraph{Capturing $\RE$ with two layers}
From \cref{semialg-parametric}, we can now deduce the two-attention-layer version of \cref{main-result-universality,main-result-undecidability}. The first
ingredient is the following version of the MRDP theorem on Diophantine sets~\citep{Mat93}:
\begin{theorem}\label{equivalence-re-diophantine}
Let $\Sigma=\{\lta_1,\ldots,\lta_m\}$. A language $L\subseteq\Sigma^*$ belongs to $\RE\cap\PI$ if and only if there is a $k\in\N$ and a polynomial $p\in\Z[X_1,\ldots,X_{m+k}]$ such that $L=\pi_{\lta_1,\ldots,\lta_m}(K)$, where
\[ K=\{w\in\{\lta_1,\ldots,\lta_{m+k}\}^* \mid p(\Parikh(w))=0 \}. \]
\end{theorem}
In other words, every language in $\RE\cap\PI$ is a projection of a language of
the form $L_p=\{w\in\{\lta_1,\ldots,\lta_m\}^*\mid p(\Psi(w))=0\}$, where
$p\in\Z[X_1,\ldots,X_m]$ is a polynomial. 
Thus, it suffices
to place $L_p$ in $\Proj(\AHATem[\le 2,U])$.
%
First observe that in 
\cref{main-result-semialgebraic}, we use one attention layer for each
multiplication, so this avenue is closed if we want to stay within two
attention layers. Instead, we use that for every polynomial
$p\in\Z[X_1,\ldots,X_m]$, there are \emph{quadratic} (i.e.\ degree $\le 2$)
polynomials $q_1,\ldots,q_r\in\Z[X_1,\ldots,X_{m+k}]$ for some $r,k\ge 0$ such
that for $\bx\in\N^m$, we have $p(\bx)=0$ if and only if there is some
$\by\in\N^k$ with $q_1(\bx,\by)=0,\ldots,q_r(\bx,\by)=0$: Just
introduce a fresh variable for each multiplication in $p$ and use the $q_i$
to assign these fresh variables. Since the language $K:=\{w\in\{\lta_1,\ldots,\lta_{m+k}\}^* \mid q_1(\Psi(w))=\cdots=q_{r}(\Psi(w))\}$ belongs to $\SemiAlg[\le 2]$ (since the $q_i$ have degree $\le 2$) and $L_p$ is a projection of $K$, this means $L_p$ belongs to $\Proj(\SemiAlg[\le 2])$. By \cref{semialg-parametric}, $\Proj(\SemiAlg[\le 2])\subseteq\Proj(\AHATem[\le 2,U])$.


\paragraph{NoPE AHAT with a single layer} The fact that two layers suffice for universality among counting properties raises the question of whether this is even possible with a single attention layer. We show here that this is not the case:
\begin{restatable}{theorem}{AHATemEqQFPA}\label{result-end-marker-one-layer}
$\AHATem[\le 1]=\AHATem[\leq1,U]=\QFPA$.
\end{restatable}
This means, with a single attention layer, NoPE-AHAT can recognize precisely
those counting properties expressible using quantifier-free Presburger
formulas.  Since satisfiability of Presburger arithmetic is well-known to be
decidable~\citep{DBLP:journals/siglog/Haase18,DBLP:conf/fsttcs/000124}, this
implies that universality and undecidability of $\AHATem$ (as we have shown for
two attention layers), do not hold with just one attention layer.  However, we
leave open whether $\SMAT$ with one attention layer have a decidable emptiness
problem.

 Before going into details, let us sketch the proof of
\cref{result-end-marker-one-layer}. For the inclusion $\AHATem[\le
1]\subseteq\QFPA$, we proceed similarly to \cref{ahat-to-semi-algebraic}, while
observing that the inequalities we have to verify are all linear inequalities:
This is because a single attention layer averages only once. Conversely, for
the inclusion $\QFPA\subseteq\AHATem[\le 1,U]$ follows easily from \cref{semialg-parametric}.
\begin{proof}[Proof of \cref{result-end-marker-one-layer}]
	We begin by proving that $\AHATem[\le 1] \subseteq \QFPA$. 
	Let $T$ be an AHAT with input embedding $\iota : \Sigma \cup \{\$\} \to \mathbb{Q}^d$, 
	a single AHA layer $\lambda$ utilising affine maps 
	$Q, K \in \mathbb{Q}^{m \times d}$, $V \in \mathbb{Q}^{k\times d}$, given as matrices, 
	and the ReLU network $\mathcal{N} : \mathbb{Q}^{d+k} \to \mathbb{Q}^e$.    
	Our goal is to construct a quantifier-free PA formula $\varphi_T$ with variables $x_i$ for $i \in \{1, \dotsc, |\Sigma|\}$ such that 
	$\Parikh^{-1}(\llbracket \varphi \rrbracket) = \{w \in \Sigma^* \mid T \text{ accepts } w\$\}$. 
	In the following, we assume  
	$\Sigma = \{a_1, \dotsc, a_m \}$ and denote $\Sigma \cup \{\$\}$ by $\Sigma'$. 
	
	First, we observe that for all words $w \in \Sigma^*$, the output of $T$ given $w\$$ is computed by
	\begin{displaymath}
	\mathcal{N}\left(\iota(\$), \frac{1}{|w\$|_{a_{i_1}} + \dotsb + |w\$|_{a_{i_h}}}
	\sum_{j=1}^h |w\$|_{a_{i_j}} V \iota(a_{i_j})\right),
	\end{displaymath}
	where $\Gamma = \{a_{i_1}, \dotsc, a_{i_h}\} \subseteq \Sigma'$
	is exactly the subset of symbols $a_{i_j}$ occurring in $w\$$ that maximise
	$\langle Q\iota(a_{i_j}), K\iota(\$) \rangle$.
	We construct $\varphi_T$ such that it mirrors exactly this computational structure.
	We have $\varphi_T = \bigvee_{\Gamma \subseteq \Sigma'} \varphi_\Gamma$, where $\bigvee$ ranges over those subsets $\Gamma$
	where $\langle Q\iota(a_{i_j}), K\iota(\$) \rangle$ is maximal for precisely the $a_{i_j} \in \Gamma$.
	The subformula $\varphi_\Gamma$ is defined as follows. 
	For now, we assume that $\$ \notin \Gamma$ and introduce some auxiliary formulas. Throughout the
	following construction steps, we assume that atomic formulas are normalised to the form
	$c_1 x_1 + \dotsb + c_n x_n \leq b$.
	
	Given the ReLU network $\mathcal{N}$, it is straightforward to construct a quantifier-free PA formula
	$\varphi^\mathcal{N}$ such that $\llbracket \varphi^\mathcal{N} \rrbracket$ exactly includes those $x_1, \dotsc, x_{d+k} \in \mathbb{N}^{d+k}$
	satisfying $\mathcal{N}(x_1, \dotsc, x_{d+k})_1 > 0$, where $\mathcal{N}(\cdot)_1$ denotes the first output dimension of $\mathcal{N}$.
	The key idea here is that the computation of a single ReLU node $v(x_1, \dotsc, x_{d+k}) = y$, with weights $c_i$ and bias $b$ of $\mathcal{N}$,
	is described by the quantifier-free PA formula:
	$(c_1x_1 + \dotsb + c_{d+k}x_{d+k} + b \leq 0 \land 0=y) \lor (c_1x_1 + \dotsb + c_{d+k}x_{d+k} + b > 0 \land c_1x_1 + \dotsb + c_{d+k}x_{d+k} + b = y)$.
	Then, by nesting this construction iteratively from the last layer to the first layer of $\mathcal{N}$, and finally replacing $= y$ with $> 0$ in the
	atomic formulas related to the first output dimension of $\mathcal{N}$, we achieve the construction of $\varphi^\mathcal{N}$. This nesting and replacement 
	also ensures that $\varphi^\mathcal{N}$ includes only the variables $x_1, \dotsc, x_{d+k}$.
	
	Let $\Gamma \subseteq \Sigma$ such that $\Gamma = \{a_{i_1}, \dotsc, a_{i_h}\}$.
	Consider the ReLU network $\mathcal{N}$, the value matrix $V$, and the embedding $\iota$.
	We construct a quantifier-free PA formula $\varphi^{\mathcal{N},V}_\Gamma$ such that
	$\llbracket \varphi^{\mathcal{N},V}_\Gamma \rrbracket$ exactly includes those
	$(x_{i_1}, \dotsc, x_{i_h}) \in \mathbb{N}^{h}$ satisfying
	$\mathcal{N}(\iota(\$), \frac{1}{x_{i_1} + \dotsb + x_{i_h}}
	\sum_{j=1}^h x_{i_j} V \iota(a_{i_j}))_1 > 0$.
	To do so, we adjust the formula $\varphi^\mathcal{N}$ as described in the following.
	To account for the fixed input $\iota(\$)$, we replace each occurrence of
	$x_1$ to $x_{d}$ in $\varphi^{\mathcal{N}}$ by the respective entry of $\iota(\$)$.
	Furthermore, to handle the specific form of the input
	$\frac{1}{x_{i_1} + \dotsb + x_{i_h}} \sum_{j=1}^h x_{i_j} V \iota(a_{i_j})$,
	we first replace each occurrence of $x_{d+l}$ with $l \in \{1, \dotsc, k\}$
	in the already modified $\varphi^\mathcal{N}$ by:
	\begin{displaymath}
	(v_{l1} \iota(a_{i_1})_1 + \dotsb + v_{ld} \iota(a_{i_1})_d)x_{i_1} +
	\dotsb +
	(v_{l1} \iota(a_{i_h})_1 + \dotsb + v_{ld} \iota(a_{i_h})_d)x_{i_h},
	\end{displaymath}
	where $v_{lj}$ are the respective entries of $V$. Lastly, we replace each
	atomic constraint $c_1 x_{i_1} + \dotsb + c_h x_{i_h} \leq b$ in the
	adjusted formula with $(c_1-b) x_{i_1} + \dotsb + (c_h-b) x_{i_h} \leq 0$
	to adjust for the factor $\frac{1}{x_{i_1} + \dotsb + x_{i_h}}$ present in the input.

	Now, we define $\varphi_\Gamma$ as $\varphi^{\mathcal{N},V,\iota}_\Gamma$. 
	If $\$ \in \Gamma$, we adjust $\varphi^{\mathcal{N},V,\iota}_\Gamma$ slightly. 
	Assuming $\$ = a_{i_j} \in \Gamma$, we replace the variable $x_{i_j}$ with the constant 
	$1$ in $\varphi^{\mathcal{N},V,\iota}_\Gamma$. Given this construction, it is clear 
	that $\Parikh^{-1}(\llbracket \varphi_T \rrbracket) = \{w \in \Sigma^+ \mid T \text{ accepts } w\$\}$, as 
	$\varphi_T$ mimics the computation of $T$ for all possible attention situations $\Gamma$.

	For the inclusion $\QFPA\subseteq\AHATem[\le, U]$, we observe that $\QFPA\subseteq\SemiAlg[\le 1]$, and thus the inclusion follows from \cref{semialg-parametric}.
\end{proof}

\section{Counting properties expressible by other models}\label{app:semilinearity}

\subsection{Semilinear counting properties}
A counting property $P \subseteq \N^d$ is said to be \defn{semilinear} if can be
defined as a boolean combination of inequalities over linear arithmetic 
expressions (over variables $x_1,\ldots,x_d$ and integer constants) and modulo
arithmetic expressions of the form $x_i \equiv a \pmod{b}$, where $a,b \in \N$ 
are fixed constants. In particular, semilinear counting properties cannot define
semialgebraic counting properties involving polynomials of degrees 2 or above.

It is also convenient to use quantifiers when defining semilinear sets. In
particular, they do not increase the expressive power since they can be
eliminated. This results in the logic called \emph{Presburger arithmetic (PA)},
which refers to the first-order theory of the structure $\langle \mathbb{N}; +,
0, 1, < \rangle$; see \cite{DBLP:journals/siglog/Haase18,DBLP:conf/fsttcs/000124}. 
\OMIT{
The \emph{quantifier-free} fragment of PA includes all PA formulas that do not 
contain any quantifiers; in other 
words, these are Boolean combinations of atomic formulas. We assume that all atomic formulas are of the form $c_1 x_1 + \dotsb + c_n x_n \leq b$ with $c_i, b \in \Q$, which we refer to 
as \emph{linear inequalities}. Additionally, we use common abbreviations such as $=$ or $>$.
Given a PA formula with $m$ free variables, meaning variables not bound by any quantifier, we denote by
$\llbracket \varphi \rrbracket$ the set of vectors in $\mathbb{N}^m$ that satisfy $\varphi$.
Let $\QFPA$ denote the class of languages $L \subseteq \Sigma^+$ for which there exists a
quantifier-free PA formula $\varphi$ with $|\Sigma|$ free variables such that 
$L=\{w\in\Sigma^+ \mid \Psi(w)\in\llbracket\varphi\rrbracket\}$.
}

\subsection{Permutation-invariant languages of LTL with counting}\label{app:ltl-semilinearity}
$\LTL[\Count]$ has the following syntax:
\begin{align*}
  \phi &::= a \mid t \leq t \mid \neg\phi \mid \phi\lor\phi \mid \X\phi \mid \phi \U \phi\\
  t    &::= k \mid k \cdot \overleftarrow{\#\phi} \mid k \cdot \overrightarrow{\#\phi} \mid t + t
\end{align*}
where $a\in\Sigma$ and $k\in\Z$. Next we define the semantics of $\LTL[\Count]$. For any word $w=a_1a_2\cdots a_\ell\in\Sigma^*$ with $a_1,a_2,\ldots,a_\ell\in\Sigma$, for each $1\leq i\leq\ell$, and each formula $\phi\in\LTL[\Count]$ we write $w,i\models\phi$ if the formula $\phi$ is satisfied in $w$ at position $i$. Formally, this relation is defined inductively as follows:
\begin{itemize}
  \item $w,i\models a$ (for $a\in\Sigma$) iff $a_i=a$,
  \item $w,i\models \neg\phi$ iff $w,i\not\models \phi$,
  \item $w,i\models \phi\lor\psi$ iff $w,i\models\phi$ or $w,i\models\psi$,
  \item $w,i\models \X\phi$ iff $i<\ell$ and $w,i+1\models\phi$,
  \item $w,i\models \phi\U\psi$ iff there is $i\leq j\leq k$ with $w,j\models\psi$ and for all $i\leq k<j$ we have $w,k\models\phi$,
  \item $w,i\models t_1\leq t_2$ iff $\sem{t_1}(w,i)\leq\sem{t_2}(w,i)$ where the semantics $\sem{t}\colon\Sigma^*\times\N\to\Z$ of a term $t$ is defined as follows: $\sem{k}(w,i)=k$, $\sem{t_1+t_2}(w,i)=\sem{t_1}(w,i)+\sem{t_2}(w,i)$, $\sem{k\cdot \overleftarrow{\#\phi}}=k\cdot|\{1\leq j<i\mid w,j\models\phi\}|$, and $\sem{k\cdot \overrightarrow{\#\phi}}=k\cdot|\{i\leq j\leq\ell\mid w,j\models\phi\}|$.
\end{itemize}

Our main result on $\LTL[\Count]$ is the following:
\begin{restatable}{theorem}{ltlCountSemilinear}\label{ltl-count-semilinear}
  Every permutation-invariant language definable in $\LTL[\Count]$ has a semilinear Parikh image.
\end{restatable}
Before we can prove \cref{ltl-count-semilinear}, we need a few more definitions. For an alphabet $\Sigma$ write $\Sigma_\varepsilon$ for the set $\Sigma\cup\{\varepsilon\}$. A ($d$-dimensional) \emph{Parikh automaton} is a tuple $\autA=(Q,\Sigma,\iota,\Delta,(C_q)_{q\in Q})$ where $Q$ is a finite set of \emph{states}, $\Sigma$ is the \emph{input alphabet}, $\iota\in Q$ is an \emph{initial} state, $\Delta\subseteq Q\times\Sigma_\varepsilon\times\N^d\times Q$ is a finite \emph{transition} relation, and $C_q\subseteq\N^d$ are semilinear \emph{target} sets. A word $w\in\Sigma^*$ is \emph{accepted} by $\autA$ if there are $a_1,a_2,\ldots,a_\ell\in\Sigma_\varepsilon$, states $q_0,q_1,\ldots,q_\ell\in Q$, and vectors $\bv_0,\bv_1,\ldots,\bv_\ell\in\N^d$ such that (i)~$q_0=\iota$ and $\bv_0=\bzero$, (ii)~for each $0\leq i<\ell$ there is a transition $(q_i,a_i,\bx_i,q_{i+1})\in\Delta$ with $\bv_{i+1}=\bv_i+\bx_i$, and (iii)~$\bv_\ell\in C_{q_\ell}$. The accepted language $L(\autA)$ of $\autA$ is the set of all words accepted by $\autA$. It is a well-known fact that for each Parikh automaton $\autA$ the accepted language $L(\autA)$ has a semilinear Parikh image. Observe that $0$-dimensional Parikh automata are essentially NFA and, hence, accept exactly the regular languages.

A \emph{Parikh transducer} is a Parikh automaton with input alphabet $\Sigma_\varepsilon\times\Gamma_\varepsilon$ where $\Sigma$ and $\Gamma$ are two alphabets. The accepted language $L(\autA)\subseteq\Sigma^*\times\Gamma^*$ of a Parikh transducer can also be seen as a map: if $(v,w)\in L(\autA)$ then we can see $v$ as the input and $w$ as the output of the transducer. Formally, for an input language $L\subseteq\Sigma^*$ a Parikh transducer computes the output $T_\autA(L)=\{w\in\Gamma^*\mid\exists v\in L\colon (v,w)\in L(\autA)\}$. If $L$ is accepted by a Parikh automaton then $T_\autA(L)$ is also accepted by a Parikh automaton. To see this, we can take the synchronized product of the Parikh automaton $\autB$ accepting $L$ and $\autA$ (i.e., $\autB$ reads the same letter from the input as $\autA$ in its first component). Accordingly, cascading of Parikh transducers is also possible, i.e., if $\autA$ and $\autB$ are Parikh transducers over $\Sigma_\varepsilon\times\Gamma_\varepsilon$ and $\Gamma_\varepsilon\times\Pi_\varepsilon$, we can also construct a Parikh transducer $\mathfrak{C}$ over $\Sigma_\varepsilon\times\Pi_\varepsilon$ computing $T_{\mathfrak{C}}=T_\autB\circ T_\autA$.

With the definition of Parikh automata and Parikh transducers we are now able to prove \cref{ltl-count-semilinear}.

\begin{proof}
  Let $\phi\in\LTL[\Count]$ be a formula such that the described language $L(\phi)$ is permutation-invariant. We will prove by induction on the structure of $\phi$ that the Parikh image of $L(\phi)$ (or actually a \emph{bounded} subset of $L(\phi)$) is semilinear. Here, a language $L\subseteq\Sigma^*$ is \emph{bounded} if there are letters $a_1,a_2,\ldots,a_n\in\Sigma$ with $L\subseteq a_1^*a_2^*\cdots a_n^*$. So, let $a_1,a_2,\ldots,a_n\in\Sigma$ be distinct letters with $\Sigma=\{a_1,a_2,\ldots,a_n\}$. Then $L(\phi)\cap a_1^*a_2^*\cdots a_n^*$ is clearly bounded and has the same Parikh image as $L(\phi)$.
  
  For each subformula $\psi$ of $\phi$ we construct a Parikh transducer that labels each position satisfying $\psi$. In the base case, we decorate each letter $a$ by $\bbb\in\{0,1\}^n$ where $\bbb[i]=1$ iff $a_i=a$. Note that this transducer handles all atomic formulas $a\in\Sigma$ at once. For $\psi=\chi_1\lor\chi_2$ we add the decoration $b\in\{0,1\}$ to each letter where $b=1$ iff one of the decorations corresponding to $\chi_1$ and $\chi_2$ is $1$. There are similar transducers (which do not introduce counters) for the cases $\psi=\neg\chi$, $\psi=\X\chi$, and $\psi=\chi_1\U\chi_2$. Note that applying these transducers to a bounded language always yields another bounded language.
  
  Now, consider a counting subformula, i.e. $\psi=\sum_{i=1}^{\ell_1} k_i\cdot \overleftarrow{\#\chi_i}+\sum_{i=\ell_1+1}^{\ell_2} k_i\cdot \overrightarrow{\#\chi_i}\leq k$. Observe that the set of positions satisfying $\psi$ is convex in the set of positions satisfying any $\chi_i$. This is true since we consider only a bounded input language. Hence, we can split the input word into three (possibly empty) intervals: (i) the positions at the beginning of the input that do not satisfy $\psi$, (ii) the positions where all positions satisfying a $\chi_i$ also satisfy $\psi$, and (iii) the positions at the end of the input that do not satisfy $\psi$. We describe in the following a Parikh transducer with $3\cdot\ell_2$ many counters - one for each of these three intervals and each formula $\chi_i$. The transducer guesses the three intervals (note that this is non-deterministic), counts positions satisfying a $\chi_i$ accordingly, decorates only the positions in the second interval labeled with a $\chi_i$ with $1$ (and everything else with a $0$), and validates in the end our choice of the intervals (via appropriate semilinear target sets ensuring that the equation in $\phi$ is not satisfied in the first and third interval and is satisfied in the second interval). Clearly, this all can be done in one (non-deterministic) Parikh transducer.
  
  Finally, we have a cascade of (Parikh) transducers decorating each position in a bounded input word with a Boolean value indicating whether $\phi$ holds in that position. If we use $a_1^*a_2^*\cdots a_n^*$ as input language for our transducers (note that this language is regular) and intersect the output with all words decorated with a $1$ in the first position, we obtain a Parikh automaton accepting exactly the language $L(\phi)\cap a_1^*a_2^*\cdots a_n^*$. Since Parikh automata accept only languages with semilinear Parikh image, we infer that $L(\phi)\cap a_1^*a_2^*\cdots a_n^*$ and, hence, $L(\phi)$ have a semilinear Parikh image.
\end{proof}



\section{Further experimental validation}
\label{app:further_experiments}
\begin{figure}[t]
    \centering
    \begin{minipage}{0.4\textwidth}
        \centering
        \begin{tabular}{cccc}
            \toprule
            $i,j$ & Val. Perf. & Test Perf. & Gen. Perf. \\
            \midrule
            1,3 & 0.016 & 0.02/0.99  & 0.03/0.99 \\
            3,2 & 0.002 & 0.003/0.99 & 0.60/0.93 \\
            3,3 & 0.001 & 0.002/0.99 & 2.26/0.85 \\
            4,2 & 0.001 & 0.001/0.99 & 0.26/0.96 \\
            5,1 & 0.004 & 0.004/0.99 & 0.03/0.99 \\
            \bottomrule
        \end{tabular}
    \end{minipage}
    \hfill
    \begin{minipage}{0.55\textwidth}
        \centering
        \begin{tikzpicture}
            \begin{axis}[
                width=\textwidth,
                height=0.6\textwidth,
                ylabel={Loss},
                xlabel={$k$},
                symbolic x coords={1,2,3,4,5},
                xtick=data,
                legend style={at={(0.5,1.2)}, anchor=north, draw=none, legend columns=-1},
                ymin=0.0005, ymax=10,
                ymode=log,
                ytick={1,0.1,0.01,0.001},
                yticklabels={$10^0$,$10^{-1}$,$10^{-2}$,$10^{-3}$},
                enlarge x limits=0.2,
                grid=major,
            ]
            \addplot+[mark=* , color=blue!50, thick, mark options={fill=blue!50}] coordinates {(1,0.016) (2,0.002) (3,0.001) (4,0.001) (5,0.004)};
            \addplot+[mark=square*, color=green!50, thick, mark options={fill=green!50}] coordinates {(1,0.02) (2,0.001) (3,0.001) (4,0.002) (5,0.004)};
            \addplot+[mark=triangle*, color=red!50, thick, mark options={fill=red!50}] coordinates {(1,0.03) (2,0.60) (3,2.1) (4,0.26) (5,0.03)};
            \legend{Val. Perf., Test Perf., Gen. Perf.}
            \end{axis}
        \end{tikzpicture}
    \end{minipage}
    \caption{Performance of softmax transformer classifiers for $L_{i,j}$ (for a selected set of $i$ and $j$ combinations).
                \textbf{Validation Performance (Val. Perf.)}: BCEWithLogitsLoss on validation data.
                \textbf{Test Performance (Test Perf.)}: BCEWithLogitsLoss and Accuracy (separated by /) on test data.
                \textbf{Generalization Performance (Gen. Perf.)}: BCEWithLogitsLoss and Accuracy (separated by /) on generalization test set.
                The y-axis uses a logarithmic scale to accommodate the different orders of magnitude in the results.}
    \label{fig:addexps}
\end{figure}

\revised{
In this section, we report additional experiments addressing a similar research question as posed in Section~\ref{sec:exps}, namely,
do softmax transformers perform well on formal languages with inherent non-linear counting properties? Therefore, we consider
the language
\begin{displaymath}
            L_{i,j} = \{a^m b^n c^{m^i n^j} \mid m,n \in \mathbb{N}^{\geq 1}\}
\end{displaymath}
for selected values of $i$ and $j$. Clearly, recognising this language requires non-linear counting capabilities. Moreover,
in contrast to $L_k$ (see Section~\ref{sec:exps}), this language poses a greater challenge in learning tasks due to its
structure (all $b$'s follow all $a$'s followed by all $c$'s) and larger alphabet size.

The experimental setup is identical to that presented in Section~\ref{sec:exps}.
The results are presented in Figure~\ref{fig:addexps} for five distinct combinations of $i$ and $j$.
Similar to our previous experiments, the table on the left shows the highest observed performance
on the validation dataset (first column) and the best performance on a balanced test dataset derived from 
the same distribution as the training and validation data (second column). This indicates that this dataset 
also contains only words of length up to 500. The final column represents another balanced test dataset
of words from length 501 to 1000, utilised to potentially reveal length generalisation performance.
The plot on the right visualises the results reported in the table.

We again observe very high performance of our trained softmax transformers on the in-distribution
test dataset (second column), which shares the same distribution as our training dataset. 
The performance generally remains high on the generalisation test set
(third column) as well. We witnessed a slight decrease compared to the results on the in-distribution
test in the case of $L_{3,3}$ (accuracy of 0.85). A general decrease 
in performance on longer inputs is expected and also witnessed in other studies (cf. \cite{framework}), but
it also indicates that focused studies are essential to reveal rigorous insights into 
the relationship between the expressibility of polynomial counting properties we established 
and their practical learnability.
}
%

\label{afterbibliography}
\newoutputstream{pagestotal}
\openoutputfile{main.pagestotal.ctr}{pagestotal}
\addtostream{pagestotal}{\getpagerefnumber{afterbibliography}}
\closeoutputstream{pagestotal}

\newoutputstream{todos}
\openoutputfile{main.todos.ctr}{todos}
\addtostream{todos}{\arabic{@todonotes@numberoftodonotes}}
\closeoutputstream{todos}

\end{document}